\def\isarxiv{1} 

\ifdefined\isarxiv
\documentclass[11pt]{article}

\usepackage[numbers]{natbib}

\else

\documentclass[10pt,twocolumn,letterpaper]{article}

\usepackage{wacv}              
\fi

\ifdefined\isarxiv

\usepackage{amsmath}
\usepackage{amsthm}
\usepackage{amssymb}
\usepackage{algorithm}
\usepackage{subfig}
\usepackage{algpseudocode}
\usepackage{graphicx}
\usepackage{grffile}
\usepackage{wrapfig,epsfig}
\usepackage{url}
\usepackage{xcolor}
\usepackage{epstopdf}
\usepackage{bbm}
\usepackage{dsfont}

\else
\usepackage{graphicx}
\usepackage{amsmath}
\usepackage{amssymb}
\usepackage{booktabs}
\usepackage{amsthm}
\usepackage{xcolor}
\usepackage{algorithm}
\usepackage{algpseudocode}
\usepackage[pagebackref,breaklinks,colorlinks]{hyperref}
\usepackage[capitalize]{cleveref}

\DeclareSymbolFont{extraup}{U}{zavm}{m}{n}
\DeclareMathSymbol{\varheart}{\mathalpha}{extraup}{86}
\DeclareMathSymbol{\vardiamond}{\mathalpha}{extraup}{87}

\fi

\allowdisplaybreaks

\ifdefined\isarxiv

\usepackage{tikz}
\usepackage{hyperref}  
\hypersetup{colorlinks=true,citecolor=blue,linkcolor=blue} 
\usetikzlibrary{arrows}
\usepackage[margin=1in]{geometry}

\else

\fi
\graphicspath{{./figs/}}

\theoremstyle{plain}
\newtheorem{theorem}{Theorem}[section]
\newtheorem{lemma}[theorem]{Lemma}
\newtheorem{definition}[theorem]{Definition}

\newtheorem{fact}[theorem]{Fact}
\newtheorem{remark}[theorem]{Remark}

\newtheorem{condition}[theorem]{Condition}

\newcommand{\wh}{\widehat}
\newcommand{\wt}{\widetilde}

\newcommand{\R}{\mathbb{R}}

\newcommand{\Lap}{\mathrm{Lap}}
\newcommand{\TLap}{\mathrm{TLap}}
\newcommand{\K}{\mathsf{K}}

\renewcommand{\hat}{\wh}

\DeclareMathOperator*{\E}{{\mathbb{E}}}

\DeclareMathOperator{\poly}{poly}

\makeatletter
\newcommand*{\RN}[1]{\expandafter\@slowromancap\romannumeral #1@}
\makeatother

\usepackage{lineno}

\ifdefined\isarxiv

\else

\crefname{section}{Sec.}{Secs.}
\Crefname{section}{Section}{Sections}
\Crefname{table}{Table}{Tables}
\crefname{table}{Tab.}{Tabs.}

\fi

\begin{document}

\ifdefined\isarxiv

\date{}

\title{Differential Privacy Mechanisms in Neural Tangent Kernel Regression}
\author{
Jiuxiang Gu\thanks{\texttt{
jigu@adobe.com}. Adobe Research.} 
\and
Yingyu Liang\thanks{\texttt{
yingyul@hku.hk}. The University of Hong Kong. \texttt{
yliang@cs.wisc.edu}. University of Wisconsin-Madison.} 
\and
Zhizhou Sha\thanks{\texttt{ shazz20@mails.tsinghua.edu.cn}. Tsinghua University.}
\and
Zhenmei Shi\thanks{\texttt{
zhmeishi@cs.wisc.edu}. University of Wisconsin-Madison.}
\and 
Zhao Song\thanks{\texttt{ magic.linuxkde@gmail.com}. The Simons Institute for the Theory of Computing at the University of California, Berkeley.}
}
\else

\title{Differential Privacy Mechanisms in Neural Tangent Kernel Regression}

\author{
  Jiuxiang Gu$^\heartsuit$\thanks{{\tt \small jigu@adobe.com}.}
  ~~
  Yingyu Liang$^{\vardiamond,\varheart}$\thanks{{\tt \small yingyul@hku.hk}. {\tt\small yliang@cs.wisc.edu}.}
  ~~
  Zhizhou Sha$^{\clubsuit}$\thanks{{\tt \small shazz20@mails.tsinghua.edu.cn}.} 
  ~~
  Zhenmei Shi$^{\vardiamond}$\thanks{{\tt \small zhmeishi@cs.wisc.edu}.}
  ~~
  Zhao Song$^{\spadesuit}$\thanks{{\tt \small magic.linuxkde@gmail.com}.}
  \\
  $^\heartsuit$Adobe Research, USA. 
  \qquad 
  $^\vardiamond$University of Wisconsin-Madison, USA. 
  \\
  $^\varheart$The University of Hong Kong, HongKong. 
  \qquad
  $^\clubsuit$Tsinghua University, China. 
  \\
  $^\spadesuit$The Simons Institute for the Theory of Computing at the University of California, Berkeley, USA. 
}

\maketitle 

\fi

\ifdefined\isarxiv
\begin{titlepage}
  \maketitle
  \begin{abstract}
  Training data privacy is a fundamental problem in modern Artificial Intelligence (AI) applications, such as face recognition, recommendation systems, language generation, and many others, as it may contain sensitive user information related to legal issues. To fundamentally understand how privacy mechanisms work in AI applications, we study differential privacy (DP) in the Neural Tangent Kernel (NTK) regression setting, where DP is one of the most powerful tools for measuring privacy under statistical learning, and NTK is one of the most popular analysis frameworks for studying the learning mechanisms of deep neural networks. In our work, we can show provable guarantees for both differential privacy and test accuracy of our NTK regression. Furthermore, we conduct experiments on the basic image classification dataset CIFAR10 to demonstrate that NTK regression can preserve good accuracy under a modest privacy budget, supporting the validity of our analysis. To our knowledge, this is the first work to provide a DP guarantee for NTK regression.   

  \end{abstract}
  \thispagestyle{empty}
\end{titlepage}

{\hypersetup{linkcolor=black}
\tableofcontents
}
\newpage

\else

\begin{abstract}
Training data privacy is a fundamental problem in modern Artificial Intelligence (AI) applications, such as face recognition, recommendation systems, language generation, and many others, as it may contain sensitive user information related to legal issues. To fundamentally understand how privacy mechanisms work in AI applications, we study differential privacy (DP) in the Neural Tangent Kernel (NTK) regression setting, where DP is one of the most powerful tools for measuring privacy under statistical learning, and NTK is one of the most popular analysis frameworks for studying the learning mechanisms of deep neural networks. In our work, we can show provable guarantees for both differential privacy and test accuracy of our NTK regression. Furthermore, we conduct experiments on the basic image classification dataset CIFAR10 to demonstrate that NTK regression can preserve good accuracy under a modest privacy budget, supporting the validity of our analysis. To our knowledge, this is the first work to provide a DP guarantee for NTK regression.   

\end{abstract}

\fi

\section{Introduction}\label{sec:intro}
Artificial Intelligence (AI) applications are widely employed in daily human life and product activities, such as face
recognition~\cite{pvz15}, recommendation systems~\cite{zyst19}, chat-based language generation~\cite{aaa+23}, and many more. 
These applications intrinsically run deep-learning models that are trained on broad datasets, where 
many contain sensitive user information, e.g., a company trains a model on its user information to provide better-customized service. 
Consequently, there is a problem with user privacy information data leakage~\cite{lgf+23}, which affects the AI company's reputation~\cite{lcl+23} and may cause severe legal issues~\cite{ylh+24,jhl+24}. 
Therefore, preserving the privacy of training data becomes a fundamental problem in deep learning. 

Differential Privacy (DP) \cite{dr14} was proposed to measure privacy rigorously and has been widely studied in many traditional statistical problems. 
Recently, many brilliant works have applied this powerful tool to machine learning settings. 
One line of work studies the classic machine learning task, e.g., \cite{rbht09} studies DP under the Support Vector Machine (SVM) model. However, these settings are still far from practical deep neural networks nowadays. 
The other line of work studies the DP in deep learning training, e.g., DP-SGD~\cite{acg+16} provides DP guarantees for the Stochastic Gradient Descent (SGD) training algorithm. 
The issue in DP training is that the trade-off between privacy guarantees and test accuracy will worsen as training time increases. 
DP training may not be practical in today's training paradigm, i.e., pre-training with an adaptation in foundation models~\cite{bha+21}, as the per-training stage may involve billions of training steps. 

To bridge the gap between practical deep learning models and practical differential privacy guarantees, in this work, we study DP Mechanisms in Neural Tangent Kernel~\cite{jgh18} (NTK) Regression. 
NTK is one of the most standard analysis frameworks for studying optimization and generalization in over-parameterized deep neural networks (DNN). 
NTK can connect the DNN training by SGD to kernel methods by using the kernel induced by gradient around the neural networks' initialization. 
Consequently, the DNN optimization can be viewed as NTK regression, and it retains almost the same generalization ability~\cite{adh+19} as kernel regression even though the DNN is in an over-parameterization regime. 

\noindent {\bf Our contributions.}
In this work, we use the ``Gaussian Sampling Mechanism'' \cite{gsy23} to add a positive semi-definite noise matrix to the neural tangent kernel matrix and the truncated Laplace mechanism \cite{dr14} to ensure the privacy of the kernel function. Then, we can show provable guarantees for both differential privacy and test accuracy of our private NTK regression, i.e.,
\begin{theorem}[Main result, informal version of Theorem~\ref{thm:main}]\label{thm:main_informal}
Under proper conditions, for any test data $x$, we have NTK-regression is $(\epsilon,\delta)$-DP and has good utility under a large probability.
\end{theorem}

Furthermore, we undertake experiments using the fundamental image classification dataset CIFAR10 to illustrate that NTK regression can maintain high accuracy with a modest privacy budget (see Figure~\ref{fig:gaussian_mtk_regression} in Section~\ref{sec:exp_main_results}). This effectively validates our analysis. To the best of our knowledge, this research is the first effort to offer a differential privacy (DP) guarantee for NTK regression.

\noindent {\bf Roadmap.}
Our paper is organized as follows. 
Section~\ref{sec:related_work} provides an overview of differential privacy, the neural tangent kernel (NTK).
Section~\ref{sec:preliminary} introduces the formal definition of DP, the definitions of both continuous and discrete versions of the NTK matrix, and the definition of NTK Regression.
In Section~\ref{sec:main_results}, we provide the privacy and utility guarantees for our private NTK regression. 
An overview of the techniques employed in this paper is discussed in Section~\ref{sec:tech_overview}.
In Section~\ref{sec:experiments}, we conduct experiments on the ten-class classification task of the CIFAR-10 dataset, demonstrating that our algorithm preserves good utility and privacy.
In Section~\ref{sec:discussion}, we thoroughly discuss several inherent intuitions behind the design of our algorithm.
Finally, we conclude in Section~\ref{sec:conclusion}.
 
\section{Related Work} \label{sec:related_work}

\noindent{\bf Differential Privacy Guarantee Analysis.}
Since the introduction of the concept of differential privacy (DP) by \cite{dmns06}, it has emerged as a crucial standard for privacy protection, both theoretically and empirically \cite{d08,llsy17,zc22,phk+23,ygz+23}. DP offers a robust and measurable privacy definition that supports the design of algorithms with specific guarantees of privacy and accuracy \cite{emn21, aimn23, ll23, hy21, gkk+23,blm+24, cem+22,emnz24,cnx22,s22,hkmn23,n22,n23,jln+19,ll24,fl22,fll24,ll23_rp,csw+23,cfl+22,dcl+24,fhs22,glw21,gll+23,llh+22,gll22,ekkl20,syyz23,csy23a,dswz23,wzz23,swyz23,gsy23,lls+24_je,lhr+24} and many more. Furthermore, innovative mechanisms have been developed beyond conventional Laplace, Gaussian, and Exponential methods \cite{dr14}. For instance, the truncated Laplace mechanism \cite{gdgk20} has been demonstrated to achieve the tightest lower and upper bounds on minimum noise amplitude and power among all $(\epsilon,\delta)$-DP distributions.

\noindent{\bf Neural Tangent Kernel.}
Numerous recent studies suggest that the analysis of optimization and generalization in deep learning should be closely integrated. NTK employs the first-order Taylor expansion to examine highly over-parameterized neural networks, from their initial states, as seen in references like \cite{ll18,mrh+18,zczg18,jgh18,dzps19,als19_neurips,zg19,os19, lxs+19,nxl+19,y19,sy19,dll+19,als19_icml,cob19,mmm19, ofls19, adh+19_icml, cg19, jt19,all19, cczg19, os20, cb20, cfw+20,zczg20,gsjw20,lss+20, bpsw21,zgj21, llwa21,swl21,  mz22,mosw22,ccbg22, gms23,qss23,sy23,gqsw24,szz24, swl24} and others. Consequently, the optimization of neural networks can be approached as a convex problem. The NTK technique has gained widespread application in various contexts, including preprocessing analysis \cite{syz21,hswz22,als+23,scl+23,ssll23,ssl24,gqsw24}, federated learning \cite{lsy23}, LoRA adaptation for large language models \cite{hwaz+21,xsw+24,smf+24,gll+24a,gls+24c,mwy+23, lss+24_multi_layer, cls+24, lss+24_relu}, and estimating scoring functions in diffusion models \cite{hrx24,wht24,gll+24b,gls+24b, wsd+23, wcz+23, wxz+24}.

\noindent{\bf Differential Privacy in Machine Learning.}
Differential privacy (DP) is a thriving and potent method with extensive applications in the field of private machine learning~\cite{lds+21}. This includes its use during the pretraining stage \cite{acg+16,phk+23}, the adaptation stage \cite{bepp22, samb24, llb+24, ynb+21,ltlh21,ssc+22}, and the inference stage \cite{ew24,lpp+20,lssz24_dp}.
Recently, \cite{gaw+22,mjw+22,zsz+24} has integrated DP into large language models, and \cite{wkw+23} applied DP into diffusion models.
DP has also been widely used in various settings,
e.g., near-neighbor counting~\cite{aimn23}, permutation hashing \cite{ll23}, BDD tree \cite{hy21}, counting tree \cite{gkk+23}, Jaccard similarity \cite{abs20} and so on. 

\section{Preliminary} \label{sec:preliminary}

In this section, we first introduce some basic notations in Section~\ref{sec:preli:notations}. 
In Section~\ref{sec:preli:dp}, we introduce the definition of DP, and the truncated Laplace mechanism. 
Then, we will introduce the definitions of the neural tangent kernel, both discrete and continuous versions in Section~\ref{sec:preli:ntk_kernel}, the definitions and key component of  NTK regression in Section~\ref{sec:preli:ntk_regression}. 

\subsection{Notations} \label{sec:preli:notations}

For any positive $n$, let $[n]$ denote the set $\{1,2,\cdots, n\}$.
For any vector $z\in\R^n$.
We define the $\ell_2$-norm of a vector $z$ as $\|z\|_2 := (\sum_{i=1}^n z_i^2)^{1/2}$, the $\ell_1$-norm as $\|z\|_1 := \sum_{i=1}^n |z_i|$, the $\ell_0$-norm as the count of non-zero elements in $z$, and the $\ell_\infty$-norm as $\|z\|_\infty := \max_{i \in [n]} |z_i|$. The transpose of vector $z$ is indicated by $z^\top$. The inner product between two vectors is denoted by $\langle \cdot, \cdot \rangle$, such that $\langle a, b \rangle = \sum_{i=1}^n a_i b_i$.

For any matrix $A \in \R^{m \times n}$. We define the Frobenius norm of $A$ as $\|A\|_F:=( \sum_{i\in [m], j\in [n]} A_{i,j}^2 )^{1/2}$.  
We use $\|A\|$ to denote the spectral/operator norm of matrix $A$.

A function $f(x)$ is said to be $L$-Lipschitz continuous if it satisfies the condition $\| f(x) - f(y) \|_2 \leq L \cdot \| x - y \|_2$ for some constant $L$.
Let ${\cal D}$ represent a given distribution. The notation $x \sim {\cal D}$ indicates that $x$ is a random variable drawn from the distribution ${\cal D}$. We employ $\E[]$ to represent the expectation operator and $\Pr[]$ to denote probability. Furthermore, we refer to a matrix as PSD to indicate that it is positive semi-definite.

As we have multiple indexes, to avoid confusion, we usually use $i,j \in [n]$ to index the training data, $s, t \in [d]$ to index the feature dimension, $r \in [m]$ to index neuron number.  

\subsection{Differential Privacy} \label{sec:preli:dp}

This section will first introduce the formal definition of differential privacy. Then, we will introduce the truncated Laplace mechanism that can ensure DP. 

\begin{definition} [Differential Privacy, \cite{dr14}] \label{def:dp}
For $\epsilon >0, \delta \geq 0$, a randomized function $\mathcal{A}$ is $(\epsilon, \delta)$-differentially private ($(\epsilon, \delta)$-DP) if for any two neighboring datasets $X \sim X'$, and any possible outcome of the algorithm $S \subset$ Range($\mathcal{A}$), $\Pr[\mathcal{A} (X) \in S] \leq e^{\epsilon} \Pr[\mathcal{A} (X') \in S] + \delta$. 
\end{definition}

Then, we introduce the sensitivity of a function $f$, which is defined to be $\Delta_f = \max_{X \sim X'} | f(X) - f(X') |$. We use $X \sim X'$ to denote two neighboring datasets. 

We use $\Lap (\lambda)$ to denote the Laplace distribution with parameter $\lambda$ with PDF $\Pr[Z = z] = \frac{1}{2 \lambda} e^{-|z| / \lambda}$.
We also use $\TLap(\Delta, \epsilon, \delta)$ to denote the Truncated Laplace distribution with PDF proportional to $e^{-|z| / \lambda}$ on the region $[-B_L, B_L]$, where $B_L = (\Delta / \epsilon) \log (1+ \frac{e^{\epsilon} - 1}{2 \delta})$. 

\begin{lemma} [Truncated Laplace Mechanism, \cite{dr14, gdgk20, aimn23}] \label{lem:truncated_laplace_mechanism}
Given a numeric function $f$ that takes a dataset $X$ as the input, and has sensitivity $\Delta$, the mechanism output $f(X) + Z$ where $Z \sim \mathrm{Lap}(\Delta / \epsilon)$ is $(\epsilon, 0)$-DP. In addition, if $Z \sim \mathrm{TLap}(\Delta, \epsilon, \delta)$, then $f(X) + Z$ is $(\epsilon, \delta)$-DP. 
\end{lemma}

Here, we introduce the critical post-processing Lemma for DP. 

\begin{lemma} [Post-Processing Lemma for DP, \cite{dr14}] \label{lem:post_processing_dp}
Let $\mathcal{M} := \mathbb{N}^{|\chi|} \rightarrow \R $ be a randomized 
algorithm that is $(\epsilon, \delta)$-differentially private. Let $f: \R \rightarrow \R'$ be an arbitrarily random mapping. Then is $f \circ \mathcal{M}: \mathbb{N}^{|\chi|} \rightarrow \R'$
$(\epsilon, \delta)$-differentially private. 

\end{lemma}

Then, we restate the composition lemma for DP. 

\begin{lemma} [Composition lemma for DP, \cite{dr14}] \label{lem:dp_composition}
Let $\mathcal{M} := \mathbb{N}^{|\chi|} \rightarrow \R $ be an $(\epsilon_i, \delta_i)$-DP algorithm for $i \in [k]$. If $\mathcal{M}_{[k]} \rightarrow \Pi_{i=1}^n \mathcal{R}_i$ satisfies $\mathcal{M}_{[k]}(x) = (\mathcal{M}_1(x), $ $\cdots , \mathcal{M}_k(x))$, then $\mathcal{M}_{[k]}$ is $(\sum_{i=1}^k \epsilon_i, \sum_{i=1}^k \delta_i)$-DP. 
\end{lemma}

\subsection{Neural Tangent Kernel} \label{sec:preli:ntk_kernel}

Then, we introduce our crucial concept,
the Neural Tangent Kernel induced by the Quadratic activation function.
We will introduce Discrete Quadratic Kernel in Definition~\ref{def:dis_quadratic_ntk} and Continuous Quadratic Kernel in Definition~\ref{def:cts_quadratic_ntk}.

{\bf Data.} 
We have $n$ training data points ${\cal D}_n = \{ (x_i, y_i)\}_{i=1}^n = (X, Y)$, where $x \in \R^d$ and $y \in \{-1, 1 \}$. We denote $X = [x_1, \dots, x_n]^\top \in \R^{n\times d}$ and $Y = [y_1, \dots, y_n]^\top \in \{-1, +1 \}^n$. We assume that $\|x_i\|_2 \le B $, $\forall i \in [n]$.

{\bf Models.}
We consider the two-layer neural network with quadratic activation function and $m$ neurons
\begin{align*}
    f(x) = \sum_{r=1}^m a_r \langle w_r, x \rangle^2,
\end{align*}
where $w_r \in \R^d$ and $a_r \in \{-1, +1\}$ for any $r \in [m]$.

\begin{definition} [Discrete Quadratic NTK Kernel] \label{def:dis_quadratic_ntk}
We draw weights $w_r \sim {\cal N}(0, \sigma^2 I_{d\times d})$ for any $r \in [m]$ and let them be fixed.
Then, we define the discrete quadratic kernel matrix $H^{\mathrm{dis}} \in \R^{n \times n}$ corresponding to ${\cal D}_n$, such that $\forall i, j \in [n]$, we have 
\begin{align*}
    H_{i, j}^{\mathrm{dis}} = \frac{1}{m} \sum_{r=1}^m \langle \langle w_r , x_i \rangle   x_i,  \langle w_r , x_j \rangle  x_j \rangle.
\end{align*}
\end{definition}

Note that $H^{\mathrm{dis}}$ is a PSD matrix, where a detailed proof can be found in Lemma~\ref{lem:h_dis_is_psd}. 

\begin{definition} [Continuous Quadratic NTK Kernel] \label{def:cts_quadratic_ntk}
We define the continuous quadratic kernel matrix $H^{\mathrm{cts}} \in \R^{n \times n}$ corresponding to ${\cal D}$, such that $\forall i, j \in [n]$, we have 
\begin{align*}
    H_{i, j}^{\mathrm{cts}} =  \E_{w \sim {\cal N}(0, \sigma^2 I_{d\times d})} \langle \langle w , x_i \rangle   x_i,  \langle w , x_j \rangle  x_j \rangle.
\end{align*}
\end{definition}

\subsection{NTK Regression} \label{sec:preli:ntk_regression}

We begin by defining the classical kernel regression problem as follows:

\begin{definition} [Classical kernel ridge regression \cite{lss+20}]
Let feature map $\phi : \R^d \to \mathcal{F}$ and $\lambda>0$ is the regularization parameter.
A classical kernel ridge regression problem can be written as
\begin{align*}
	\min_{w \in \R^{n}} \frac{1}{2}\| Y - \phi(X)^\top w \|_2^2 + \frac{1}{2}\lambda\| w \|_2^2.
\end{align*}
\end{definition}

Then, we are ready to introduce the NTK Regression problem as follows:

\begin{definition} [NTK Regression \cite{lss+20}] \label{def:ntk_regression}
If the following conditions are met:
\begin{itemize}
    \item Let $\mathsf{K}(\cdot, \cdot): \R^{d} \times \R^{d} \to \R$ be a kernel function, i.e., $\mathsf{K}(x, z) =  \frac{1}{m} \sum_{r=1}^m \langle \langle w_r , x \rangle   x,  \langle w_r , z \rangle  z \rangle, \forall x, z \in \R^d.$
    \item Let $K \in \R^{n \times n}$ be the kernel matrix with $K_{i,j} = \mathsf{K}(x_i,x_j)$, $\forall i, j \in [n] \times [n]$.
    \item Let  $\alpha \in \R^n$ be the solution to $(K + \lambda I_n) \alpha = Y$. Namely, we have $\alpha = (K + \lambda I_n)^{-1} Y$. 
\end{itemize}

Then, for any data $x \in \R^d$, the NTK Kernel Regression can be denoted as 
\begin{align*}
    f_K^*(x) = \frac{1}{n} \mathsf{K}(x,X)^\top \alpha.
\end{align*}
\end{definition}

\section{Main Results} 
\label{sec:main_results}

In this section, we will introduce several essential lemmas that form the basis of our main result. 

Firstly, we provide a high-level overview of our intuition. Recall that in the definition of the NTK regression, we have $\alpha = (K + \lambda I)^{-1} Y$ and $\K(x, X)$ (see also Definition~\ref{def:ntk_regression}). We aim to protect the sensitive information in the training data $X \in \R^{n \times d}$. Therefore, we only need to ensure the privacy of $\alpha$ and $\K(x, X)$. 

To ensure the privacy of $\alpha$, we initially focus on privatizing $K$. Subsequently, we demonstrate that $\alpha$ remains private by applying the post-processing lemma of DP ( Lemma~\ref{lem:post_processing_dp}).
Privatizing $K$ is non-trivial, as $K = H^{\mathrm{dis}}$, indicating that $K$ is a positive semi-definite (PSD) matrix (Lemma~\ref{lem:h_dis_is_psd}). We denote $\wt{K}$ as the privacy-preserving counterpart of $K$. $\wt{K}$ must maintain the PSD property, a condition that classical mechanisms such as the Laplace Mechanism and Gaussian Mechanism cannot inherently guarantee a private matrix.
In the work by \cite{gsy23}, the ``Gaussian Sampling Mechanism" addresses this challenge, ensuring that the private version of the Attention Matrix also retains PSD.

Then, we use the truncated Laplace mechanism to ensure the privacy of $X$. Then, by post-processing lemma, we have the privacy guarantees for $\K(x, X)$. 

Finally, with the help of the composition lemma of DP, we can show that the entire NTK regression is also DP. 

So far, we have introduced the high-level intuition of our entire algorithm. Then, we will dive into the details. 
Firstly, we introduce the definition of the neighboring datasets.
\begin{definition} [$\beta$-close neighbor dataset, \cite{gsy23}] \label{def:beta_ntk_neighbor_dataset}
Let $B > 0$ be a constant. 
Let $n$ be the number of data points.
Let dataset ${\cal D} = \{(x_i, y_i)\}_{i=1}^n$, where $x_i \in \R^d$ and $\|x_i\|_2 \le B$ for any $i \in [n]$.
We define ${\cal D'}$ as a neighbor dataset with one data point replacement of ${\cal D}$. Without loss of generality, we have ${\cal D'} = \{(x_i, y_i)\}_{i=1}^{n-1} \cup \{(x_n', y_n)\}$.Namely, we have ${\cal D}$ and ${\cal D'}$ only differ in the $n$-th item. 

Additionally, we assume that $x_n$ and $x_n'$ are $\beta$-close. Namely, we have
\begin{align*}
    \| x_n - x_n' \|_2 \leq \beta.
\end{align*}

\end{definition}

\subsection{DP Guarantees for NTK Regression}

In this section, we will state the DP property of the entire NTK regression using the composition and post-processing lemma of DP. The corresponding lemma is as follows. 

\begin{lemma} [DP guarantees for NTK regression
, informal version of Lemma~\ref{lem:ntk_regerssion_dp}] \label{lem:ntk_regerssion_dp:informal}
Let $\K(x, \wt{X})$ be defined as Lemma~\ref{lem:DP_for_KxX:informal}.
Let $(\wt{K} + \lambda I)^{-1}$ be defined as Lemma~\ref{lem:results_of_the_aussian_sampling_mechanism}.
Let $\epsilon_X, \delta_X \in \R$ denote the DP parameter for $\K(x, X)$.
Let $\epsilon_{\alpha}, \delta_{\alpha} \in \R$ denote the DP parameter for $(K+\lambda I)^{-1}$.
Let $\epsilon = \epsilon_X + \epsilon_{\alpha}, \delta = \delta_X + \delta_{\alpha}$. 
Then, we can show that the private NTK regression (Algorithm~\ref{alg:main}) is $(\epsilon, \delta)$-DP. 
\end{lemma}

Basically, we can easily prove this lemma by using the composition lemma of DP. For further details of the proof, please refer to Section~\ref{sec:app:ntk_regression_dp} in the appendix. 

\subsection{Utility Guarantees for NTK Regression}

In this section, we provide the utility guarantees for the private NTK regression introduced in the previous section. 

\begin{lemma} [Utility guarantees for NTK regression, informal version of Lemma~\ref{lem:ntk_regression_utility}]\label{lem:ntk_regression_utility:informal}
Let $\Delta_X = \sqrt{d} \cdot \beta$. 
Let $\epsilon_X, \delta_X \in \R$ denote the DP parameters for $X$. 
Let $B_L =  (\Delta_X / \epsilon_X) \log (1+ \frac{e^{\epsilon_X} - 1}{2 \delta_X})$. 
If all conditions hold in Condition~\ref{cond:utility_cond}, 
then, with probability $1 - \gamma$, we have 
\begin{align*}
    | f_{K}^*(x) - f_{\wt{K}}^*(x) | \leq 
    O(\frac{B^3 \sqrt{d} B_L}{\eta_{\min} + \lambda} + \frac{ \rho \cdot \eta_{\max} \cdot \omega}{(\eta_{\min} + \lambda)^2 }). 
\end{align*}

\end{lemma}

\subsection{Main Theorem} \label{sec:main_result:main_theorem}

Then, we are ready to introduce our main result, including both the privacy and utility guarantees of our private NTK regression (Algorithm~\ref{alg:main}). 

\begin{theorem}[Private NTK regression
]\label{thm:main}
Let $\Delta_X = \sqrt{d} \cdot \beta$. 
Let $\epsilon_X, \delta_X \in \R$ denote the DP parameters for $X$.
Let $B_L =  (\Delta_X / \epsilon_X) \log (1+ \frac{e^{\epsilon_X} - 1}{2 \delta_X})$. 
If all conditions hold in Condition~\ref{cond:dp_condition},  Condition~\ref{cond:psd_sensitivity_m}, and Condition~\ref{cond:utility_cond}, 
then, for any test data $x$, with probability $1 - \delta_3 - \gamma$, we have that $f_{\wt{K}}^*(x)$ is $(\epsilon,\delta)$-DP and \begin{align*}
    | f_{K}^*(x) - f_{\wt{K}}^*(x) | \leq 
    O(\frac{B^3 \sqrt{d} B_L}{\eta_{\min} + \lambda} + \frac{ \rho \cdot \eta_{\max} \cdot \omega}{(\eta_{\min} + \lambda)^2 }). 
\end{align*}
\end{theorem}
The proof of this theorem follows from directly combining the DP guarantees of NTK regression 
(Lemma~\ref{lem:ntk_regerssion_dp:informal}) and the utility guarantees of NTK regression (Lemma~\ref{lem:ntk_regression_utility:informal}). 

\section{Technical Overview} \label{sec:tech_overview}

In Section~\ref{sec:tech_overview:key_concepts}, we introduce two crucial concepts used for the Gaussian sampling mechanism. 
In Section~\ref{sec:gaussain_sampling_mechanism}, we will adhere to the framework established in \cite{gsy23} and elaborate on the functioning of the ``Gaussian Sampling Mechanism," along with its primary results and requirements.
In Section~\ref{sec:utility_of_gaussian_sampling_mechanism}, we will examine the utility implications of employing the "Gaussian Sampling Mechanism." Additionally, we include a remark that analyzes the trade-off between privacy and utility inherent in our approach.
In Section~\ref{sec:tech_overview:KxX_dp}, we introduce our privacy guarantee results on the kernel function $\K(x, X)$, which is achieved by the truncated Laplace mechanism. 
In Section~\ref{sec:tech_overview:KxX_utility}, we provide a detailed analysis of the utility of the private kernel function $\K(x, \wt{X})$. 

\subsection{Key Concepts} \label{sec:tech_overview:key_concepts}

In this section, we will introduce the two essential definitions $M$ and $\Delta$ used in the privacy proof of ``Gaussian Sampling Mechanism", which need to satisfy $M < \Delta$, which is also the {\bf Condition 4} in Condition~\ref{cond:dp_condition}. We will begin by presenting the definition of $M$. 

\begin{definition} [Definition of $M$, \cite{gsy23}] \label{def:m}
Let  $\mathcal{M}: (\R^n)^d \to \R^{n \times n}$ be a (randomized) algorithm that given a dataset of $d$ points in $\R^n$ outputs a PSD matrix. 
Let $\mathcal{Y}, \mathcal{Y}' \in  (\R^n)^d$.  
Then, we define 
\begin{align*}
    M := \| \mathcal{M}(\mathcal{Y})^{1/2}\mathcal{M}(\mathcal{Y}^{'})^{-1}\mathcal{M}(\mathcal{Y})^{1/2}-I  \|_F .
\end{align*}
\end{definition}

Afterward, we proceed to define $\Delta$. 

\begin{definition} [ Definition of $\Delta$, \cite{gsy23} ] \label{def:delta}
If we have the following conditions:
\begin{itemize}
    \item Let $\epsilon \in (0, 1)$ and $\delta \in (0, 1)$.
    \item Let $k$ denote the number of i.i.d. samples $g_1,g_2,\cdots,g_k$ from $\mathcal{N}(0,\Sigma_1)$ output by Algorithm~\ref{alg:the_gaussian_sampling_mechanism}. 
\end{itemize}
We define
\begin{align*}
    \Delta := \min\bigg\{ \frac{\epsilon}{\sqrt{8k \log(1/\delta)}},\frac{\epsilon}{8 \log(1/\delta)} \bigg\} .
\end{align*}
\end{definition}

\subsection{DP Guarantees for \texorpdfstring{$(K+\lambda I)^{-1}$}{}}
\label{sec:gaussain_sampling_mechanism}

In this section, we recapitulate the analytical outcomes of the ``Gaussian Sampling Mechanism'' as presented in Theorem~\ref{lem:results_of_the_aussian_sampling_mechanism} from \cite{gsy23}. The associated algorithm is detailed in Algorithm~\ref{alg:the_gaussian_sampling_mechanism}.

Firstly, we outline the conditions employed in the ``Gaussian Sampling Mechanism'' as follows:

\begin{condition} \label{cond:dp_condition}
We need the following conditions for DP:

\begin{itemize}
    \item {\bf Condition 1.} Let $\epsilon_{\alpha} \in (0,1)$, $\delta_{\alpha} \in (0,1)$, $k \in \mathbb{N}$.
    \item {\bf Condition 2.} Let $\mathcal{Y},\mathcal{Y}'$ denote neighboring datasets, which differ by a single data element.
    \item {\bf Condition 3.} Let $\Delta$ be defined in Definition~\ref{def:delta} and $ \Delta < 1$.
    \item {\bf Condition 4.} Let $M,{\cal M}$ be defined in Definition~\ref{def:m} and $M \leq \Delta$.
    \item {\bf Condition 5.} Let the input $\Sigma = \mathcal{M}(\mathcal{Y})$.
    \item {\bf Condition 6.} Let $\rho = O( \sqrt{ ( n^2+\log(1/\gamma) )  / k }+ ( n^2+\log(1/\gamma) ) /{k} )$.
\end{itemize}

\end{condition}

Prior to delving into the primary analysis of the ``Gaussian Sampling Mechanism'' , we offer a succinct overview of its underlying intuition. As noted at the outset of Section~\ref{sec:main_results}, the task of obtaining a private positive semi-definite (PSD) matrix is non-trivial.

Nevertheless, by leveraging covariance estimation within the ``Gaussian Sampling Mechanism'', we can guarantee that the estimated matrix will remain PSD. This is because for any $i \in [k]$, we have $g_i g_i^\top$ is PSD matrix, then $k^{-1} \sum_{i=1}^k g_i g_i^\top$ is also PSD matrix. 

With this foundation, we are ready to introduce the analysis of the ``Gaussian Sampling Mechanism’’. The analysis is presented as follows:

\begin{lemma}[DP guarantees for $(K+\lambda I)^{-1}$,
Theorem 6.12 in \cite{gsy23}, Theorem 5.1 in \cite{akt+22}, informal version of Lemma~\ref{lem:results_of_the_aussian_sampling_mechanism:formal}]
\label{lem:results_of_the_aussian_sampling_mechanism}
If all conditions hold in Condition~\ref{cond:dp_condition} and Condition~\ref{cond:psd_sensitivity_m},
then, there exists an Algorithm~\ref{alg:the_gaussian_sampling_mechanism} such that
\begin{itemize}
    \item Part 1. Algorithm~\ref{alg:the_gaussian_sampling_mechanism} is $(\epsilon_{\alpha}, \delta_{\alpha})$-DP.
    \item Part 2. Outputs $\hat{\Sigma} \in \mathbb{S}_+^n$ denotes the private version of input $\Sigma$, such that with probabilities at least $1-\gamma$,
    \begin{align*}
        \| \Sigma^{-1/2} \wh{\Sigma} \Sigma^{-1/2}-I_n \|_F \leq \rho.
    \end{align*}   
    \item Part 3. 
    $
       (1-\rho) \Sigma \preceq \wh{\Sigma} \preceq (1+\rho)  \Sigma.
    $
\end{itemize}
\end{lemma}

In Lemma~\ref{lem:results_of_the_aussian_sampling_mechanism}, {\bf Part 1} claims the privacy guarantees of the ``Gaussian Sampling Mechanism'', {\bf Part 2} establishes the critical properties necessary to ensure the utility of the ``Gaussian Sampling Mechanism'', and {\bf Part 3} presents the ultimate utility outcomes of the algorithm.

Note that in our setting, we use $\Sigma = K$, where $K$ is non-private Discrete Quadratic NTK Matrix in NTK Regression, and we also have $\wh{\Sigma} = \wt{K}$, where $\wt{K}$ denotes the private version of $K$.

\begin{algorithm}
    \caption{The Gaussian Sampling Mechanism, \cite{gsy23} }\label{alg:the_gaussian_sampling_mechanism}
    \begin{algorithmic}[1]
        \Procedure{Algorithm}{$\Sigma,k$}
        \State PSD matrix $\Sigma \in \R^{n \times n}$ and parameter $k \in \mathbb{N}$
        \State Obtain vectors $g_1,g_2,\cdots,g_k$ by sampling $g_i \sim \mathcal{N}(0,\Sigma)$, independently for each $i \in [k]$
        \State Compute $\hat{\Sigma}=\frac{1}{k}\sum_{i=1}^k g_i g_i^\top$ 
        \Comment{Covariance estimate.}
        \State \Return $\hat{\Sigma}$ 
        \EndProcedure
    \end{algorithmic}
\end{algorithm}

We need the following conditions so that we can make {\bf Condition 4} in Condition~\ref{cond:dp_condition} hold, with probability $1-\delta_3$.   
See the detailed proof in Section~\ref{sec:sensitivity_of_psd_matrix}. 
\begin{condition} \label{cond:psd_sensitivity_m}
We need the following conditions for the calculation of $M$ (see Definition~\ref{def:m}). 
\begin{itemize}
    \item {\bf Condition 1.} If ${\cal D} \in \R^{n \times d}$ and ${\cal D}' \in \R^{n \times d}$ are neighboring dataset (see Definition~\ref{def:beta_ntk_neighbor_dataset})
    \item {\bf Condition 2.} Let $H^{\mathrm{dis}}$ denote the discrete NTK kernel matrix generated by ${\cal D}$, and ${H^{\mathrm{dis}}}'$ denotes the discrete NTK kernel matrix generated by neighboring dataset ${\cal D}'$.
    \item {\bf Condition 3.} Let $H^{\mathrm{dis}} \succeq \eta_{\min} I_{n \times n}$, for some $\eta_{\min} \in \R$. 
    \item {\bf Condition 4.} Let $\beta = O(\eta_{\min} / \poly(n, \sigma, B))$, where $\beta$ is defined in Definition~\ref{def:beta_ntk_neighbor_dataset}. 
    \item {\bf Condition 5.} Let $\psi := O (\sqrt{n} \sigma^2 B^3 \beta )$. 
    \item {\bf Condition 6.} Let $\delta_1, \delta_2 , \delta_3 \in (0, 1)$. Let $\delta_1 = \delta_2 / \poly(m)$. Let $\delta_2 = \delta_3 / \poly(n)$.  
    \item {\bf Condition 7.} Let $d = \Omega (\log (1 / \delta_1))$.  
    \item {\bf Condition 8.} Let $m = \Omega(n \cdot d B^2 \beta^{-2} \log (1 / \delta_2))$.  
\end{itemize}
\end{condition}

\subsection{Utility Guarantees for \texorpdfstring{$(K+\lambda I)^{-1}$}{}}
\label{sec:utility_of_gaussian_sampling_mechanism}

In this section, we will provide utility guarantees under ``Gaussian Sampling Mechanism". By Lemma~\ref{lem:k_lambda_inverse_utility:informal}, we will argue that, ``Gaussian Sampling Mechanism" provides good utility under differential privacy. 

We start with introducing the necessary conditions used in proving the utility of ``Gaussian Sampling Mechanism". 

\begin{condition} \label{cond:utility_cond}
    We need the following conditions for Utility guarantees of ``Gaussian Sampling Mechanism":
    
\begin{itemize}
    \item {\bf Condition 1.} If ${\cal D} \in \R^{n \times d}$ and ${\cal D}' \in \R^{n \times d}$ are neighboring dataset (see Definition~\ref{def:beta_ntk_neighbor_dataset})
    \item {\bf Condition 2.} Let $H^{\mathrm{dis}}$ denote the discrete NTK kernel matrix generated by ${\cal D}$ (see Definition~\ref{def:dis_quadratic_ntk}).
    \item {\bf Condition 3.} Let $ \eta_{\max} I_{n \times n} \succeq H^{\mathrm{dis}} \succeq \eta_{\min} I_{n \times n}$, for some $\eta_{\max}, \eta_{\min} \in \R$.
    \item {\bf Condition 4.} Let $\wt{H}^{\mathrm{dis}}$ denote the private $H^{\mathrm{dis}}$ generated by Algorithm~\ref{alg:the_gaussian_sampling_mechanism} with $H^{\mathrm{dis}}$ as the input. 
    \item {\bf Condition 5.} Let $K = H^{\mathrm{dis}}, \wt{K} = \wt{H}^{\mathrm{dis}}$ in Definition~\ref{def:ntk_regression}. Then we have $f_{K}^*(x)$ and $f_{\wt{K}}^*(x)$. 
    \item {\bf Condition 6.} Let $\sqrt{n} \psi / \eta_{\min} < \Delta$, where $\Delta$ is defined in Definition~\ref{def:delta}. 
    \item {\bf Condition 7.} Let $\rho = O( \sqrt{ ( n^2+\log(1/\gamma) )  / k }+ ( n^2+\log(1/\gamma) ) /{k} )$. 
    \item {\bf Condition 8.} Let $\omega :=  6 d \sigma^2 B^4$. 
    \item {\bf Condition 9.} Let $\gamma \in (0, 1)$. 
\end{itemize}

\end{condition}

We then leverage {\bf Part 3} of Lemma~\ref{lem:results_of_the_aussian_sampling_mechanism} to derive the error between the outputs of the private and non-private NTK Regression, thereby demonstrating the utility of our algorithm.

\begin{lemma} [Utility guarantees for $(K+\lambda I)^{-1}$, informal version of Lemma~\ref{lem:k_lambda_inverse_utility}] \label{lem:k_lambda_inverse_utility:informal}
If all conditions hold in Condition~\ref{cond:utility_cond},
then, with probability $1 - \gamma$, we have 
\begin{align*}
    \|(K + \lambda I)^{-1} -  (\wt{K} + \lambda I)^{-1} \| \leq 
    O( \frac{\rho \cdot \eta_{\max}}{(\eta_{\min} + \lambda)^2}) 
\end{align*}
\end{lemma}

The interplay between privacy and utility guarantees is complex. Our algorithm exhibits a property akin to that of other classical differential privacy algorithms: an increase in privacy typically results in a decrease in utility, and conversely.
We will provide a thorough explanation of the privacy-utility trade-off in the subsequent Remark.

\begin{algorithm}[!ht]\caption{ Private NTK Regression
}\label{alg:main}
\begin{algorithmic}[1]
\Procedure{Main}{$X \in \R^{n \times d}, m,k$} \Comment{Theorem~\ref{thm:main}} 
    \State Draw $w_1, \cdots, w_m \in \R^d$ random Gaussian vectors
    \State Compute $K$ such such $K_{i,j} = \mathsf{K}(x_i,x_j)$ 
    \State Obtain vectors $g_1, g_2, \cdots, g_k$ by sampling $g_i \sim {\cal N}(0, K)$ independently for each $i \in [k]$
    \State Compute $\wt{K} \gets \frac{1}{k} \sum_{i=1}^k g_i g_i^\top$ \Comment{Lemma~\ref{lem:results_of_the_aussian_sampling_mechanism}}
    \State Compute $\wt{X} \gets X + \TLap(\Delta_X, \epsilon_X, \delta_X)$ \Comment{Lemma~\ref{lem:DP_for_KxX:informal}}
    \State Compute $f_{\wt{K}}^*(x) \gets \mathsf{K}(x, \wt{X})^\top ( \wt{K} + \lambda \cdot I_n )^{-1} Y$
    \State \Return $f_{\wt{K}}^*(x)$ 
\EndProcedure
\end{algorithmic}
\end{algorithm}

\begin{remark} [Trade-off between Privacy and Utility in Lemma~\ref{lem:k_lambda_inverse_utility:informal}] \label{rem:privacy_utility_trade_off}

An inherent trade-off exists between the privacy and utility guarantees of our algorithm. Specifically, enhancing privacy typically results in a degradation of utility. 
Recall that the variable $k$ represents the number of sampling iterations in the Gaussian Sampling Mechanism.

To ensure privacy, we must adhere to {\bf Condition 4} as outlined in Condition~\ref{cond:dp_condition}, which requires that $M < \Delta$. Here, $M$ is a constant defined in Definition~\ref{def:m}, with its precise value calculated in Lemma~\ref{lem:sensitivity_from_spectral_to_F}. In contrast, $\Delta$ is defined in Definition~\ref{def:delta} and is dependent on the value of $k$. Consequently, to achieve stronger privacy, namely a smaller DP parameter $\epsilon_{\alpha}$ or $\delta_\alpha$, it is necessary to decrease $k$ to meet the $M < \Delta$ constraint.

On the other hand, for utility considerations, as defined by $\rho = O( \sqrt{ ( n^2+\log(1/\gamma) ) / k } + ( n^2+\log(1/\gamma) ) / k )$, a reduction in $k$ results in an increase in $\rho$. This, in turn, leads to diminished utility.
\end{remark}

Due to the limitation of space, we refer the readers to Lemma~\ref{lem:k_lambda_inverse_utility} in the appendix for the details of proof of Lemma~\ref{lem:k_lambda_inverse_utility:informal}.
A detailed explanation of the trade-off between privacy and utility can be found in Remark~\ref{rem:privacy_utility_trade_off}.

\subsection{DP Guarantees for \texorpdfstring{$\K(x, X)$}{}} \label{sec:tech_overview:KxX_dp}

Then, we will introduce how to ensure the DP property of the kernel function $\K(x, X)$ by using the truncated Laplace mechanism. 

\begin{lemma} [DP guarantees for $\K(x, X)$, informal version of Lemma~\ref{lem:DP_for_KxX}] \label{lem:DP_for_KxX:informal}
If the following conditions hold:
\begin{itemize}
    \item Let $x \in \R^d$ denote an arbitrary query. 
    \item Let $\epsilon_X, \delta_X \in \R$ denote the DP parameters. 
    \item Let $\Delta_X := \sqrt{d} \cdot \beta$ denote the sensitivity of $X$. 
    \item Let $\K(x, X)$ be defined as Definition~\ref{def:ntk_regression}. 
    \item Let $\wt{X} := X + \TLap(\Delta_X, \epsilon_X, \delta_X)$ denote the private version of $X$, where $\wt{X}$ is $(\epsilon_X, \delta_X)$-DP. 
\end{itemize}

Then, we can show that $\K(x, \wt{X})$ is $(\epsilon_X, \delta_X)$-DP. 
\end{lemma}

To sum up, we first use the truncated Laplace mechanism to ensure the $(\epsilon_X, \delta_X)$-DP on $\wt{X}$. Then, we use the post-processing lemma to ensure the privacy of $\K(x, X)$. More details can be found in Section~\ref{sec:app:Kxx_dp}. 

\subsection{Utility Guarantees for \texorpdfstring{$\K(x, X)$}{}} \label{sec:tech_overview:KxX_utility}

The utility analysis for the private kernel function $\K(x, \wt{X})$ is as follows. 

\begin{lemma} [Utility guarantees for $\K(x, X)$, informal version of Lemma~\ref{lem:KxX_utility}] \label{lem:KxX_utility:informal}
If the following conditions hold:
\begin{itemize}
    \item Let $x \in \R^d$ be a query, where for some $B \in \R$, $\| x \|_2 \leq B$. 
    \item Let $\K(x, X) \in \R^n$ be defined as Definition~\ref{def:ntk_regression}.
    \item Let $\wt{X} \in \R^{n \times d}$ be defined as Lemma~\ref{lem:DP_for_X}. 
    \item Let $\Delta_X = \sqrt{d} \cdot \beta$. 
    \item Let $\epsilon_X, \delta_X \in \R$ denote the DP parameters for $X$. 
    \item Let $B_L =  (\Delta_X / \epsilon_X) \log (1+ \frac{\exp(\epsilon_X) - 1}{2 \delta_X})$. 
\end{itemize}

Then, we can show that
\begin{align*}
    \| \K(x, \wt{X}) - \K(x, X) \|_2 \leq 2 \sqrt{n} B^3 \sqrt{d} B_L.
\end{align*}
\end{lemma}

\section{Experiments} \label{sec:experiments}

This section will introduce the experimental methodology employed on the CIFAR-10 dataset. The corresponding results are visualized in Fig.~\ref{fig:gaussian_mtk_regression}.
In Section~\ref{sec:experiment_setup}, we enumerate all the parameters and the experimental setup we utilized.
Section~\ref{sec:exp_main_results} presents a detailed analysis of the outcomes from our experiments.

\subsection{Experiment Setup} \label{sec:experiment_setup}

\noindent{\bf Dataset.} Our experiments are conducted on the CIFAR-10 dataset \cite{kh09}, which comprises ten distinct classes of colored images, including subjects such as airplanes, cats, and dogs. The dataset is partitioned into 50,000 training samples and 10,000 testing samples, with each image measuring $32 \times 32$ pixels and featuring RGB channels. 
Although NTK regression is initially a binary classification model, we can extend it to ten classification classes. 
To be more specific, let $n_{\mathrm{cls}}$ denote the number of classes. Here, we have $n_{\mathrm{cls}} = 10$. Then, we have $Y \in \R^{n \times n_{\mathrm{cls}}}$, which denotes the labels of the training data. Hence, we have $\alpha \in \R^{n \times n_{\mathrm{cls}}}$. During the query, for each query $x \in \R^d$, we will have a prediction $p_{\mathrm{pred}} \in \R^{n_{\mathrm{cls}}}$ by the NTK regression. Then, we apply argmax to $p_{\mathrm{pred}}$, and we will get the final predicted label of the query $x$. 
We randomly choose $1,000$ images for training and $100$ for testing for each class. Namely, we will have $10,000$ in training images and $1,000$ in test images. 

\noindent{\bf Feature Extraction.} CIFAR-10 images possess a high-dimensional nature ($32 \times 32 \times 3 = 3,072$ dimensions), which poses a challenge for NTK Regression. To address this, we leverage the power of ResNet \cite{hzrs16} to reduce the dimension. Following the approach of \cite{j18}, we employ ResNet-18 to encode the images and extract features from the network's last layer, yielding a 512-dimensional feature representation for each image.

\noindent{\bf Feature Normalization.} Prior to training our NTK Regression, we normalize all image features such that each feature vector's $\mathcal{L}_2$ norm is equal to $1$.

\noindent{\bf NTK Regression Setup.} For both NTK Regression and the NTK Regression Kernel Matrix, we select $m = 256$ neurons and a random Gaussian variance of $\sigma = 1$. This means that for each $r \in [m]$, the weights $w_r$ are drawn from the normal distribution ${\cal N}(0, I_{d\times d})$. Additionally, we set the regularization factor $\lambda = 10$.

\subsection{Experiment Results Analysis} 
\label{sec:exp_main_results}

Following the experimental setup detailed in Section~\ref{sec:experiment_setup}, we present the results in Fig.~\ref{fig:gaussian_mtk_regression}.

During the execution of NTK Regression, we initially compute $H^{\mathrm{dis}}$ (as defined in Definition~\ref{def:dis_quadratic_ntk}) based on the quadratic activation between the training data and $m$ neurons. As $H^{\mathrm{dis}}$ is a symmetric matrix, it is also positive semi-definite.
Then, in accordance with the notation in Definition~\ref{def:ntk_regression}, we define $K = H^{\mathrm{dis}}$. We then apply the Gaussian Sampling Mechanism, as described in Section~\ref{sec:main_results}, to privatize $K$, denoting the private version as $\wt{K}$. Due to the properties of the Gaussian Sampling Mechanism, $\wt{K}$ remains symmetric and thus positive semi-definite. Lemma~\ref{lem:results_of_the_aussian_sampling_mechanism} guarantees that $\wt{K}$ is $(\epsilon_{\alpha}, \delta_{\alpha})$-DP.

We then compute the private $\alpha$, denoted as $\wt{\alpha}$, by $\wt{\alpha} = (\wt{K} + \lambda I_{n \times n})^{-1} Y$. By the Post-processing Lemma of differential privacy (refer to Lemma~\ref{lem:post_processing_dp}), we confirm that $\wt{\alpha}$ is also $(\epsilon_{\alpha}, \delta_{\alpha})$-DP.

Subsequently, we privatize the kernel function $\K(x, X)$. As described in Section~\ref{sec:tech_overview:KxX_dp}, we apply truncated Laplace noise on $X$ to get the private version $\wt{X}$, which is $(\epsilon_X, \delta_X)$-DP. Then, by the post-processing lemma, for any query $x \in \R^d$, we have $\K(x, \wt{X})$ is $(\epsilon_X, \delta_X)$-DP. 

Consequently, applying the composition lemma, we can have the private NTK regression is $(\epsilon, \delta)$-DP.

In our experiment, we fix the differential privacy parameter $\delta = 2 \times 10^{-3}$.
We recall that $\Delta$ is defined in Definition~\ref{def:delta}, and $M$ is defined in Definition~\ref{def:m}. To satisfy \noindent{\bf Condition 5} in Condition~\ref{cond:dp_condition}, we must ensure $M < \Delta$.

We select $k$ to be greater than or equal to $8 \cdot \log(1 / \delta)$, ensuring that for any $\epsilon > 0$, the condition $\epsilon / \sqrt{8 k \log(1 / \delta)} \leq \epsilon / (8 \log(1 / \delta))$ holds. Consequently, we have $\Delta = \epsilon / \sqrt{8 k \log(1 / \delta)}$ (see also Definition~\ref{def:delta}).

Under this setup, the condition $M < \Delta$ is equivalent to:
\begin{align} \label{eq:k_upper_bound_beta}
    k \leq \frac{\epsilon^2 \eta_{\min}^2}{8 \log(1 / \delta) n^2 \sigma^4 B^8 \beta^2}
\end{align}

In our experimental setup, we have $\sigma = 1$, $B = 1$, $\eta_{\min} = 7 \times 10^{-3}$, and $n = 10^3$. We assume $\beta = 10^{-6}$. We then compute the upper bound for $k$ using Eq.~\eqref{eq:k_upper_bound_beta} and adhere to this upper bound when conducting our experiments. The outcomes are presented in Fig.~\ref{fig:gaussian_mtk_regression}.

\begin{figure}[!ht]
\centering
\includegraphics[width=0.45\textwidth]{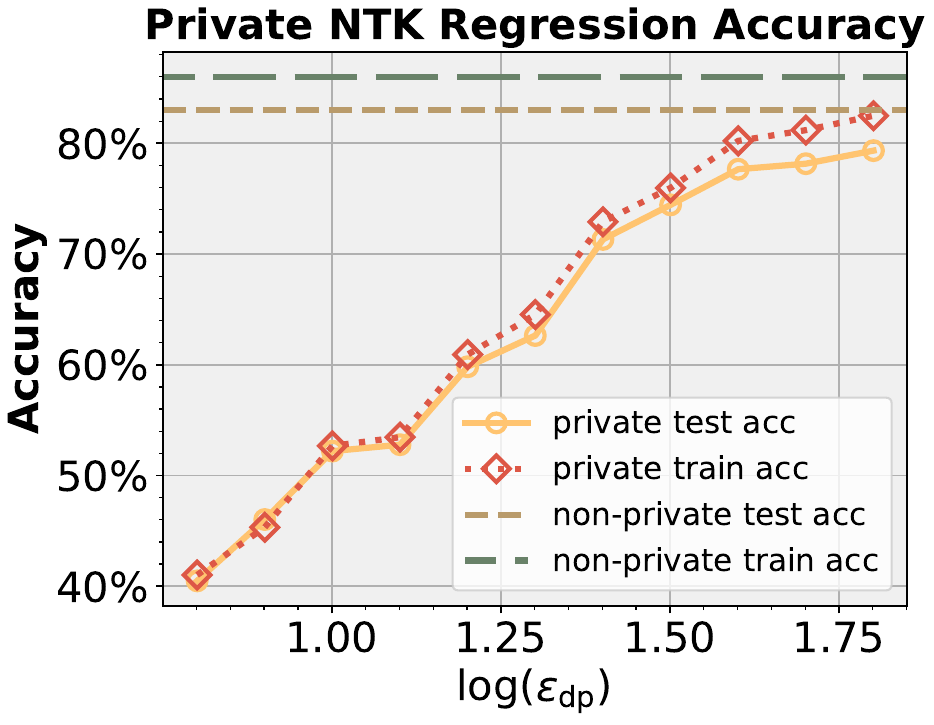}
\caption{
The trade-off between the accuracy parameter and privacy parameter.
We conduct experiments on different privacy budget $\epsilon$, where we fixed the $\delta = 2 \times 10^{-3}$, and we assume that $\beta = 10^{-6}$ in our experiments. The x-axis denotes the $\log (\epsilon_{\mathrm{dp}})$, where the $\log$ denotes $\log_{10}$. The y-axis denotes the binary classification accuracy. As privacy budget $\epsilon_{\mathrm{dp}}$ increase, both private test acc and private train acc approach to non-private train acc and non-private test acc, respectively. 
}
\label{fig:gaussian_mtk_regression}
\end{figure}

\section{Discussion} \label{sec:discussion}

\noindent {\bf DP in kernel and gradient. }
In DP-SGD~\cite{acg+16}, they add Gaussian noise on the gradient for privacy. 
As they are a first-order algorithm, their function sensitivity is more robust for single-step training. 
However, as discussed in Section~\ref{sec:intro}, to guarantee DP for the whole training process, their DP Gaussian noise variance will increase as $T$ becomes larger (see Theorem 1 in~\cite{acg+16}). The DP-SGD is not practical when $T$ is too large. 
On the other hand, our NTK regression setting is a second-order algorithm involving kernel matrix inverse.  
Then, our key technical issues are (1) introducing a PSD noise matrix to keep kernel PSD property and (2) using $L_2$ regularization to make the kernel sensitivity more robust (see more details in Section~\ref{sec:tech_overview}).

\noindent {\bf Where to add noise?}
In the work, we add noise both on the kernel function $\K(x, X)$ and the the $\alpha$ to make the entire NTK regression private. 

Others may argue that if we can only add noise on $\K(x, X)$ or $\alpha$ to ensure the privacy of the NTK regression. However, we argue that this is not feasible. The reasons are as follows. 
The primary reason is that we need to apply DP's post-processing lemma to ensure NTK regression's privacy. 
Therefore, we need to ensure the privacy of all the inputs we cared of the NTK regression. 
Since the NTK regression can be viewed as a function $\mathsf{F}(X, Y, K)$, which takes $X, Y$ and $K$ (the NTK matrix) as the inputs. 
Since we only aim to protect the sensitive information in $X$, we can view the NTK regression as a function $\mathsf{F}(X, K)$ only takes $X$ and $K$ as the inputs. 
Hence, we need to ensure privacy both on $\K(x, X)$ (corresponds to the input $X$) and $\alpha$ (corresponds to the input $K$) to have the privacy guarantees for the entire NTK regression. 

\noindent {\bf Why NTK rather than Neural Networks (NNs)?}
We elucidate our preference for NTK-regression over traditional NNs based on two primary aspects. (1) Traditional NNs present analytical challenges. (2) NTK-regression effectively emulates the training dynamics of overparameterized NNs, facilitating a more tractable analysis.

To begin with, the analysis of traditional NNs is far from straightforward. Most modern NNs incorporate a variety of non-linear activation functions, complicating the derivation of theoretical bounds such as sensitivity and utility bounds. Consequently, the simplistic bounds for NNs are often impractically loose. Additionally, the intricacies of the training process for NNs, which lacks a closed-form solution or guaranteed global optimality in stochastic gradient descent (SGD), render a thorough analysis exceptionally difficult.

In contrast, the analysis of NTK regression serves as an good starting point. The NTK captures the essence of overparameterized NNs' behavior. The kernel function's and the linear properties in NTK regression allow for the derivation of closed-form solutions for its constituent parts, significantly simplifying analysis. 

Consequently, this study adopts NTK regression as its analytical foundation, deferring a detailed examination of traditional NNs to our future research.

\section{Conclusion}\label{sec:conclusion}
In conclusion, we have presented the first DP guarantees for NTK regression. 
From the theoretical side, we provide differential privacy guarantees for the NTK regression, and the theoretical utility bound for the private NTK regression. 
From the experimental side, we conduct validation experiments on the ten-class classification task on the CIFAR-10 dataset, which demonstrates our algorithm preserves good utility under a small private budget.  
This work opens new avenues for privacy-preserving deep learning using an NTK-based algorithm.

\ifdefined\isarxiv
\section*{Acknowledgement}
Research is partially supported by the National Science Foundation (NSF) Grants 2023239-DMS, CCF-2046710, and Air Force Grant FA9550-18-1-0166.
\bibliographystyle{alpha}
\bibliography{ref}
\else
\bibliography{ref}
\bibliographystyle{plain}
\fi

\newpage
\onecolumn
\appendix
\begin{center}
	\textbf{\LARGE Appendix }
\end{center}

\ifdefined\isarxiv
\else
{\hypersetup{linkcolor=black}
\tableofcontents
}
\fi

\paragraph{Roadmap.} The Appendix is organized as follows. 

In Section~\ref{appendix:sec:basic_tools}, we introduce the fundamental probability, algebra, and differential privacy tools utilized throughout the paper.
In Section~\ref{sec:app:proof_overview}, we briefly introduce how we prove the sensitivity of the NTK matrix $H^{\mathrm{dis}}$. 
Section~\ref{appendix:sec:psd} contains the proof of the positive semi-definite (PSD) property for the Discrete Quadratic NTK.
Section~\ref{appendix:sec:gaussian_sampling_mechanism} restates the analysis results for the "Gaussian Sampling Mechanism."
Section~\ref{appendix:sec:sensitivity_for_ntk} presents a comprehensive proof of the sensitivity for both the Continuous Quadratic NTK and the Discrete Quadratic NTK. Subsequently, we ensure the privacy of $(K+\lambda I)^{-1}$, which further ensures the privacy of the $\alpha$ parameter of the NTK regression.
In Section~\ref{appendix:sec:utility_guarantees}, we discuss the utility guarantees $(K+\lambda I)^{-1}$.
In Section~\ref{sec:app:Kxx_dp}, we prove the DP guarantees for the kernel function $\K(x, X)$ of the NTK regression. 
In Section~\ref{sec:app:KxX_utility}, we analyze the utility guarantees for the private $\K(x, X)$. 
In Section~\ref{sec:app:ntk_regression_dp}, combining the DP results of $(K+\lambda I)^{-1}$ and $\K(x, X)$, we introduce the privacy guarantees for the entire NTK regression. 
Finally, in Section~\ref{sec:app:ntk_regression_utility}, we show the utility guarantees for the entire NTK regression.

\section{Basic Tools} \label{appendix:sec:basic_tools}

In this section, we display more fundamental concepts and tools for a better understanding of the readers. 
In Section~\ref{appendix:sec:basic_tools:probability_tools}, we introduce a useful tail bound for the Chi-square distribution, as well as the classical concentration inequality.
In Section~\ref{appendix:sec:basic_tools:algebra_tools}, we demonstrate useful properties of Gaussian distribution and some useful inequalities for matrix norm and vector norm. 

\subsection{Probability Tools} \label{appendix:sec:basic_tools:probability_tools}

We state some standard tools from the literature. 
Firstly, we will state the tail bound of the Chi-Squared distribution in the following Lemma. 
\begin{lemma}[Chi-square tail bound, Lemma 1 in \cite{lm00} 
]\label{lem:chi_square_bound}
    Let $X \sim \mathcal{X}_k^2$ be a chi-squared distributed random variable with $k$ degrees of freedom. Each one has zero means and $\sigma^2$ variance. 
    
    Then, it holds that
    \begin{align*}
        \Pr[X - k\sigma^2 \geq (2\sqrt{k t} + 2t) \sigma^2]
        \leq & ~ \exp{(-t)} \\
        \Pr[k\sigma^2 - X \geq 2\sqrt{k t}\sigma^2]
        \leq & ~ \exp{(-t)}
    \end{align*}
\end{lemma}

Then, we will present several useful concentration inequalities. 
\begin{lemma}[Bernstein inequality (scalar version) \cite{b24}]\label{lem:scalar_bernstein}
Let $X_1, \cdots, X_n$ be independent zero-mean random variables. Suppose that $|X_i| \leq M$ almost surely, for all $i$. Then, for all positive $t$,
\begin{align*}
\Pr [ \sum_{i=1}^n X_i > t ] \leq \exp ( - \frac{ t^2/2 }{ \sum_{j=1}^n \E[X_j^2]  + M t /3 } )
\end{align*}
\end{lemma}

\begin{lemma}[Markov's inequality]\label{lem:markov_ineq}
Let $x$ be a non-negative random variable and $t > 0$, and let $f(x)$ be the probability density function (pdf) of $x$. Then, we have
\begin{align*}
    \Pr[x \geq t] \leq \frac{\E[x]}{t}.
\end{align*}
\end{lemma}

\subsection{Basic Algebra Tools} \label{appendix:sec:basic_tools:algebra_tools}

We state a standard fact for the 4-th moment of Gaussian distribution.
\begin{fact} \label{fact:fourth_moment_of_gaussian}
Let $x \sim {\cal N}(0,\sigma^2)$, then it holds that $\E_{x \sim {\cal N}(0,\sigma^2)}[x^4] = 3 \sigma^2$.
\end{fact}

We also list some facts related to matrix norm and vector norm here.
\begin{fact}\label{fac:norm}
We have
\begin{itemize}
    \item Let $A\in \R^{n \times n}$, then we have $\| A \|_{F} \leq \sqrt{n} \| A \|$.
    \item Let $A \in \R^{n \times n}$, then we have $\| A \| \leq \| A \|_F$
    \item For two vectors $a,b \in \R^n$, then we have $ | a b^\top | \leq \| a \|_2 \cdot \| b \|_2$
\end{itemize}
\end{fact}
\section{Proof Overview} \label{sec:app:proof_overview}

In this section, we offer a succinct overview of the proof techniques employed throughout this paper. 

We begin by detailing the method for calculating the Sensitivity of the Continuous Quadratic NTK in Section~\ref{sec:tech_overview:cts_sensitivity}.
Subsequently, in Section~\ref{sec:tech_overview:brigh_gap_between_cts_and_dis}, we utilize concentration inequalities, specifically Bernstein's Inequality, to narrow the gap between the Continuous Quadratic NTK and Discrete Quadratic NTK.
In Section~\ref{sec:sensitivity_of_psd_matrix}, we will delve into the sensitivity of the PSD matrix $K$, which is the matrix we aim to privatize.
Concluding this overview, we present the approaches for establishing privacy guarantees and utility guarantees in Sections~\ref{sec:tech_overview:privacy_guarantees} and~\ref{sec:tech_overview:utility_guarantees}, respectively.

\subsection{Sensitivity of Continuous Quadratic NTK} \label{sec:tech_overview:cts_sensitivity}

Given that the Discrete Quadratic NTK is defined as $H_{i, j}^{\mathrm{dis}} = \frac{1}{m} \sum_{r=1}^m \langle \langle w_r , x_i \rangle   x_i,  \langle w_r , x_j \rangle  x_j \rangle$ (see also Definition~\ref{def:dis_quadratic_ntk}), analyzing its sensitivity when a single data point is altered in the training dataset poses a challenge. In contrast, the Continuous Quadratic NTK, which incorporates Expectation in its definition (refer to Definition~\ref{def:cts_quadratic_ntk}), is significantly much easier to analyze. The discrepancy between the Discrete and Continuous versions of the NTK can be reconciled through Concentration inequalities. Consequently, we will start our analysis of NTK Regression by focusing on the Continuous Quadratic NTK.

In the ``Gaussian Sampling Mechanism" (see also Section~\ref{sec:gaussain_sampling_mechanism}), we defined the concept of neighboring datasets being $\beta$-close (refer to Definition~\ref{def:beta_ntk_neighbor_dataset}). We focus on examining the Lipschitz property of individual entries within the Continuous Quadratic NTK.

Without loss of generality, let us consider a dataset $\mathcal{D}$ of length $n$, and let neighboring datasets $\mathcal{D}$ and $\mathcal{D'}$ differ solely in their $n$-th data point. Consequently, as per the definition of the Continuous Quadratic NTK, the respective kernels $H^{\mathrm{cts}}$ and ${H^{\mathrm{cts}}}'$ will differ exclusively in their $n$-th row and $n$-th column.

To examine the Lipschitz property for the $n$-th row and $n$-th column, we must consider two distinct cases. Initially, we focus on the sole diagonal entry, the $(n, n)$-th element, in the Continuous Quadratic NTK. The Lipschitz property for this entry is given by (see also Lemma~\ref{lem:quadratic_ntk_diagonal_single_lipschitz}):
\begin{align*}
    | H^{\mathrm{cts}}_{n, n} - {H^{\mathrm{cts}}_{n, n}}' | \leq 4 \sigma^2 B^3 \cdot  \| x - x' \|_2
\end{align*}

Subsequently, we will address the remaining $2n - 2$ non-diagonal entries within the $n$-th row and $n$-th column, for which the Lipschitz property is as follows: (see also Lemma~\ref{lem:quadratic_ntk_non_diagonal_single_lipschitz}). For all $\{ 
(i, j) : (i = n, j \neq n) or (i \neq n, j = n), i, j \in [n] \}$, we have 
\begin{align*}
    | H^{\mathrm{cts}}_{i,j} - {H^{\mathrm{cts}}_{i,j}}' | \leq 2 \sigma^2 B^3 \cdot \| x - x' \|_2
\end{align*}

Moreover, we compute the sensitivity of the Continuous Quadratic NTK with respect to $\beta$-close neighboring datasets. Leveraging the Lipschitz properties of the diagonal and non-diagonal entries discussed earlier, we derive the sensitivity of the Continuous Quadratic NTK as follows: (see also Lemma~\ref{lem:beta_cts_quadratic_sensitivity})
\begin{align*}
    \| H^{\mathrm{cts}} - {H^{\mathrm{cts}}}' \|_F \leq O( \sqrt{n} \sigma^2 B^3 \beta)
\end{align*}

\subsection{Bridge the Gap between Continuous and Discrete Quadratic NTKs} \label{sec:tech_overview:brigh_gap_between_cts_and_dis}

In the previous section, we established the sensitivity of the Continuous Quadratic NTK. Here, we aim to bridge the divide between the Continuous Quadratic NTK and Discrete Quadratic NTK.

The underlying rationale is that the number of neurons, $m$, in the Discrete Quadratic NTK can be extremely large. With the help of this property, we invoke the strength of concentration inequalities, specifically Bernstein's Inequality, to demonstrate that with high probability, the discrepancy between the Continuous  Quadratic NTK and Discrete Quadratic NTK is negligible. Namely, we have (see also Lemma~\ref{lem:union_bound_of_h_dis_h_cts_f_norm})
\begin{align*}
    \|H^{\mathrm{dis}} - H^{\mathrm{cts}}\|_F \leq n \epsilon
\end{align*}

Further, we examine the sensitivity of the Discrete Quadratic NTK with respect to $\beta$-close neighboring datasets. Recall that we have previously established the sensitivity of the Continuous Quadratic NTK in Lemma~\ref{lem:beta_cts_quadratic_sensitivity}. We now incorporate both Bernstein's Inequality and the sensitivity of the Continuous Quadratic NTK into our analysis. We have (see also Lemma~\ref{lem:beta_choice_of_eps})
\begin{align*}
    \| H^{\mathrm{dis}} - {H^{\mathrm{dis}}}' \|_F \leq O (\sqrt{n} \sigma^2 B^3 \beta)
\end{align*}

\subsection{Sensitivity of PSD Matrix} \label{sec:sensitivity_of_psd_matrix}

In alignment with \cite{gsy23}, this section will present Lemma~\ref{lem:sensitivity_from_spectral_to_F}, which offers a calculation for $M$ (refer to Definition~\ref{def:m}). Utilizing this result and noting that $\Delta$ is dependent on $k$ (see Definition~\ref{def:delta}), we can subsequently refine the sampling time $k$ to satisfy {\bf Condition 4} as mentioned in Condition~\ref{cond:dp_condition}.

\begin{lemma} [Sensitivity of PSD Matrix $H^{\mathrm{dis}}$, informal version of Lemma~\ref{lem:sensitivity_from_spectral_to_F:formal}] \label{lem:sensitivity_from_spectral_to_F}
If all conditions hold in Condition~\ref{cond:psd_sensitivity_m}, then, with probability $1 - \delta_3$, we have
\begin{itemize}
\item Part 1.
\begin{align*}
  \|  (H^{\mathrm{dis}})^{-1/2}  {H^{\mathrm{dis}}}' (H^{\mathrm{dis}})^{-1/2} - I \| \leq \psi / \eta_{\min}
\end{align*}
\item Part 2.
\begin{align*}
     \|  (H^{\mathrm{dis}})^{-1/2}  {H^{\mathrm{dis}}}' (H^{\mathrm{dis}})^{-1/2} - I \|_F \leq \sqrt{n} \psi / \eta_{\min}
\end{align*}
\end{itemize}
\end{lemma}
In Lemma~\ref{lem:sensitivity_from_spectral_to_F}, {\bf Part 1} and {\bf Part 2} respectively establish the bounds for $M$ (as defined in Definition~\ref{def:m}) under the spectral norm and Frobenius norm. These findings offer insights into how to modify $k$ in order to fulfill {\bf Condition 4} as required by Condition~\ref{cond:dp_condition}.

\subsection{Privacy Guarantees for Private NTK Regression} \label{sec:tech_overview:privacy_guarantees}

In Lemma~\ref{lem:results_of_the_aussian_sampling_mechanism} and under {\bf Condition 4.} in Condition~\ref{cond:dp_condition}, we need to compute $M$ (as defined in Definition~\ref{def:m}) for the Discrete Quadratic NTK.

The pivotal approach is to first utilize the sensitivity of the Discrete Quadratic NTK (refer to Lemma~\ref{lem:beta_choice_of_eps}) to establish the inequalities for positive semi-definite matrices (see Lemma~\ref{lem:sensitivity}). Specifically, we obtain
\begin{align*}
     (1 - \psi / \eta_{\min}) H^{\mathrm{dis}} \preceq {H^{\mathrm{dis}}}' \preceq (1 + \psi / \eta_{\min}) H^{\mathrm{dis}}
\end{align*}

Building on these inequalities, we can further demonstrate that (as shown in Lemma~\ref{lem:sensitivity_from_spectral_to_F})
\begin{align*}
     \|  (H^{\mathrm{dis}})^{-1/2}  {H^{\mathrm{dis}}}' (H^{\mathrm{dis}})^{-1/2} - I \|_F \leq \sqrt{n} \psi / \eta_{\min}
\end{align*}

Consequently, since $\psi = O (\sqrt{n} \sigma^2 B^3 \beta )$, by choosing a small $\beta$, we achieve
\begin{align*}
     M = \|  (H^{\mathrm{dis}})^{-1/2}  {H^{\mathrm{dis}}}' (H^{\mathrm{dis}})^{-1/2} - I \|_F \leq \Delta
\end{align*}

which fulfills the requirements of {\bf Condition 4.} in Condition~\ref{cond:dp_condition}.
Thus, by Lemma~\ref{lem:results_of_the_aussian_sampling_mechanism}, we establish the privacy guarantees for our algorithm.

\subsection{Utility Guarantees for Private NTK Regression} \label{sec:tech_overview:utility_guarantees}

In this section, we will present a comprehensive analysis of the utility guarantees for our private NTK Regression.

Recall that in the definition of NTK Regression (refer to Definition~\ref{def:ntk_regression}), we have
\begin{align*}
    f_K^*(x) = \mathsf{K}(x,X)^\top (K + \lambda I_n)^{-1} Y
\end{align*}

We commence by establishing a bound on the spectral norm of $(K + \lambda I_n)^{-1}$ as follows (see also Lemma~\ref{lem:k_lambda_inverse_utility:informal})
\begin{align*}
    \|(K + \lambda I)^{-1} -  (\wt{K} + \lambda I)^{-1} \| \leq O( \frac{\rho \cdot \eta_{\max}}{(\eta_{\min} + \lambda)^2})
\end{align*}

Here, we employ the Moore-Penrose inverse (see also Lemma~\ref{lem:inverse_minus_bound}) to address the inversion in $(K + \lambda I_n)^{-1}$.

Utilizing the spectral norm bound, we then derive bounds for the $\mathcal{L}_2$ norms of $\mathsf{K}(x,X)$ and $Y$. After selecting a small $\beta$, we arrive at the final result that (see also Lemma~\ref{lem:k_lambda_inverse_utility:informal})
\begin{align*}
    | f_{K}^*(x) - f_{\wt{K}}^*(x) | \leq
    O( \frac{\rho \cdot \eta_{\max} \cdot \omega}{(\eta_{\min} + \lambda)^2})
\end{align*}

Thus, we have demonstrated that our private NTK Regression maintains high utility while simultaneously ensuring privacy.

Similar to other differential privacy algorithms, our algorithm encounters a trade-off between privacy and utility, where increased privacy typically results in a reduction in utility, and conversely. An in-depth examination of this trade-off is provided in Remark~\ref{rem:privacy_utility_trade_off}.

\section{Positive Semi-Definite Matrices} \label{appendix:sec:psd}

This section introduces the proof for the PSD property of the Discrete Quadratic NTK Matrix (see Definition~\ref{def:dis_quadratic_ntk}). 

\begin{lemma} [$H^{\mathrm{dis}}$ is PSD] \label{lem:h_dis_is_psd}
Let $H^{\mathrm{dis}}$ denote the discrete quadratic NTK matrix in Definition~\ref{def:dis_quadratic_ntk}. 

Then, we can show that $H^{\mathrm{dis}}$ is PSD. 
\end{lemma}

\begin{proof}
Recall in Definition~\ref{def:dis_quadratic_ntk}, we have $H^{\mathrm{dis}} \in \R^{n \times n}$, where the $(i, j)$-th entry of $H^{\mathrm{dis}}$ satisfies
\begin{align*}
    H_{i, j}^{\mathrm{dis}} = \frac{1}{m} \sum_{r=1}^m \langle \langle w_r , x_i \rangle   x_i,  \langle w_r , x_j \rangle  x_j \rangle.
\end{align*}

Let $b_{r, i} = \langle w_r , x_i \rangle   x_i$, where for any $i \in [n]$, $b_{r, i} \in \R^d$.

Let  $B_r = [b_{r, 1}, b_{r, 2}, \cdots , b_{r, n}]$, where $B_r \in \R^{n \times n}$. 

Then, we have
\begin{align*}
    H^{\mathrm{dis}} = \frac{1}{m} \sum_{r = 1}^m B_r B_r^\top 
\end{align*}

Since $B_r B_r^\top$ is PSD matrix, $H^{\mathrm{dis}}$ is also PSD matrix. 

\end{proof}

\section{Gaussian Sampling Mechanism}
\label{appendix:sec:gaussian_sampling_mechanism}

In this section, we restate the analysis for ``Gaussian Sampling Mechanism", which guarantees the privacy of our algorithm, and provides potential tools for demonstrating its utility. 

\begin{lemma}[
DP guarantees for $(K+\lambda I)^{-1}$,
Theorem 6.12 in \cite{gsy23}, Theorem 5.1 in \cite{akt+22}, formal version of Lemma~\ref{lem:results_of_the_aussian_sampling_mechanism}]
\label{lem:results_of_the_aussian_sampling_mechanism:formal}
If we have the below conditions,

\begin{itemize}
    \item {\bf Condition 1.} Let $\epsilon_{\alpha} \in (0,1)$, $\delta_{\alpha} \in (0,1)$, $k \in \mathbb{N}$.
    \item {\bf Condition 2.} Neighboring datasets $\mathcal{Y},\mathcal{Y}'$ differ in a single data element.
    \item {\bf Condition 3.} Let $\Delta$ be denoted as Definition~\ref{def:delta} and $ \Delta < 1$.
    \item {\bf Condition 4.} Let $M,{\cal M}$ be denoted as Definition~\ref{def:m} and $M \leq \Delta$.
    \item {\bf Condition 5.} An input $\Sigma = \mathcal{M}(\mathcal{Y})$.
    \item {\bf Condition 6.} $\rho = O( \sqrt{ ( n^2+\log(1/\gamma) )  / k }+ ( n^2+\log(1/\gamma) ) /{k} )$.
\end{itemize}
    
Then, there exists an Algorithm~\ref{alg:the_gaussian_sampling_mechanism} such that
\begin{itemize}
    \item Part 1. Algorithm~\ref{alg:the_gaussian_sampling_mechanism} is $(\epsilon_{\alpha}, \delta_{\alpha})$-DP.
    \item Part 2. Outputs $\hat{\Sigma} \in \mathbb{S}_+^n$ such that with probabilities at least $1-\gamma$,
    \begin{align*}
        \| \Sigma^{-1/2} \wh{\Sigma} \Sigma^{-1/2}-I_n \|_F \leq \rho
    \end{align*}   
    \item Part 3. 
    \begin{align*}
       (1-\rho) \Sigma \preceq \wh{\Sigma} \preceq (1+\rho)  \Sigma.
    \end{align*}  
\end{itemize}
\end{lemma}

\section{Sensitivity of Neural Tangent Kernel}
\label{appendix:sec:sensitivity_for_ntk}

In this section, we provide proof of the sensitivity of the NTK Kernel Matrix. 
In Section~\ref{sec:cts_ntk_kernel_sensitivity}, we demonstrate the sensitivity of Continuous Quadratic NTK under $\beta$-close neighboring dataset. 
Then, in Section~\ref{sec:dis_ntk_kernel_sensitivity}, we use concentration inequalities to bridge the gap between Continuous Quadratic NTK and Discrete Quadratic NTK. Further we prove the sensitivity of the Discrete Quadratic NTK under $\beta$-close neighboring dataset. 
Based on previous Section, we introduce the calculation for $\|  (H^{\mathrm{dis}})^{-1/2}  {H^{\mathrm{dis}}}' (H^{\mathrm{dis}})^{-1/2} - I \|_F$ in Section~\ref{sec:sensitivity_calculation}, which plays a critical role in satisfying the requirements of privacy guarantees. 

\subsection{Sensitivity of Continuous Quadratic NTK} \label{sec:cts_ntk_kernel_sensitivity} 

We first introduce the fundamental property of the nested inner product. 

\begin{lemma} [Nested inner product property] \label{lem:nested_inner_prod_property}
For any $w, a, b \in \mathbb{R}^d$, we have
\begin{align*}
    \langle \langle w , a \rangle a, \langle w , b \rangle b \rangle = w^\top a  a^\top b b^\top w.
\end{align*}
\end{lemma}

\begin{proof}
\begin{align*} 
    \langle \langle w , a \rangle a, \langle w , b \rangle b \rangle
    = & ~ \langle \langle w , a \rangle  a, b \langle b, w \rangle \rangle \notag \\
    = & ~ \langle w^\top a  a, b b^\top w \rangle \notag \\
    = & ~ w^\top a  a^\top b b^\top w
\end{align*}
where the 1st step is because of $\langle w , b \rangle \in \mathbb{R}$, the 2nd step is due to $\langle w , a \rangle = w^\top a$, the 3rd step is from basic property of inner product. 
\end{proof}

Then, we are ready to consider the Lipschitz property for each entry in the Continuous Quadratic NTK Matrix. Since diagonal entries and off-diagonal entries have different Lipschitz properties, we need to consider these two cases separately. 

We first consider the off-diagonal case. 

\begin{lemma} [Lipschitz property for single entry (off-diagonal entries)] \label{lem:quadratic_ntk_non_diagonal_single_lipschitz}
If we have the below conditions,
\begin{itemize}
    \item Let $B > 0$ be a constant.
    \item Let $n$ be the number of data points.
    \item Let $m$ be the number of neurons. 
    \item Let dataset ${\cal D} = (X, Y)$, where $X \in \R^{n \times d}$ and $Y \in \R^n$. 
    \item Let $x_i \in \R^d$ denote $X(i, *)$, and $\| x_i \|_2 \leq B$, for any $i \in [n]$. 
    \item Let ${\cal D'}$ be the neighbor dataset (see Definition~\ref{def:beta_ntk_neighbor_dataset}).
    Without loss of generality, we assume that ${\cal D}$ and ${\cal D'}$ only differ in the $n$-th item. 
    \item Let $x := x_n \in {\cal D}$ and  $x' := x_n' \in {\cal D'}$. 
    \item Let $H^{\mathrm{cts}}$ and ${H^{\mathrm{cts}}}'$ be defined as Definition~\ref{def:cts_quadratic_ntk}, where $H_{i, j}^{\mathrm{cts}} =  \E_{w} \langle \langle w , x_i \rangle   x_i,  \langle w , x_j \rangle  x_j \rangle$. And ${H^{\mathrm{cts}}}' \in \R^{n \times n}$ is the kernel corresponding to ${\cal D'}$.
\end{itemize}

Then, we can show, for all $\{ 
(i, j) : (i = n, j \neq n) or (i \neq n, j = n), i, j \in [n] \}$, we have 
\begin{align*}
    | H^{\mathrm{cts}}_{i,j} - {H^{\mathrm{cts}}_{i,j}}' | \leq 2 \sigma^2 B^3 \cdot \| x - x' \|_2
\end{align*}
\end{lemma}

\begin{proof}
Without loss of generality, we only consider the $i = n, j \neq n$ case. The other case follows by symmetry. 

Let $v = x_j \in \R^d$. 

Then, we have the following
\begin{align*}
    | H^{\mathrm{cts}}_{n,j} - {H^{\mathrm{cts}}_{n,j}}' | = & ~ |\E_{w} [\langle \langle w , v \rangle  v,  \langle w , x \rangle  x \rangle] - \E_{w} [\langle \langle w , v \rangle  v,  \langle w , x' \rangle  x' \rangle]| \\
    = & ~ | \E_{w} [\langle \langle w , v \rangle  v, \langle w , x \rangle  x \rangle - \langle \langle w , v \rangle  v, \langle w , x' \rangle x' \rangle] |
\end{align*}
where the 1st step is because of the definition of $H^{\mathrm{cts}}_{i, j}$, the 2nd step is due to the property of expectation. 

Let $A := v v^\top (x  x^\top -  x' x'^\top) \in \mathbb{R}^{d \times d}$. Let $A(s, t) \in \mathbb{R}$ denotes the $(s, t)$-th entry of $A$.

Let $w(s) \in \mathbb{R}$ denote the $s$-th entry of $w \in \mathbb{R}^d$.

Further, we can calculate
\begin{align*}
    | \E_{w} [\langle \langle w , v \rangle  v, \langle w , x \rangle  x \rangle - \langle \langle w , v \rangle  v, \langle w , x' \rangle x' \rangle] |
    = & ~ | \E_{w} [ w^\top v  v^\top (x x^\top - x' x'^\top) w ] | \\
    = & ~ | \E_{w} [  w^\top A w ] | \\
    = & ~ | \E_{w}  [\sum_{s=1}^d w(s)^2 A(s, s)] + \E_{w} [\sum_{s \neq t} w(s) w(t) A(s, t)]  | \\
    = & ~ | \E_{w} [ \sum_{s=1}^d w(s)^2 A(s, s)  ] | \\
    = & ~ | \E_{w} [ \sum_{s=1}^d w(s)^2 (v v^\top (x  x^\top -  x' x'^\top))(s, s) ] | 
\end{align*}
where the 1st step is because of
Lemma~\ref{lem:nested_inner_prod_property}
, the 2nd step is due to definition of $A$, the 3rd step is from we separate the diagonal term and off-diagonal term, the 4th step comes from $\E[w(s) w(t)] = \E[w(s)] \E[w(t)] =0$ (when $w(s)$ and $w(t)$ are independent), the fifth step follows from the definition of $A$.

Let $x(s) \in \mathbb{R}$ denotes the $s$-th entry of $x \in \mathbb{R}^d$. Let $x(s)' \in \mathbb{R}$ denotes the $s$-th entry of $x' \in \mathbb{R}^d$. 
Let $v(s) \in \mathbb{R}$ denotes the $s$-th entry of $v \in \mathbb{R}^d$. 

Then, we consider the $\E_{w} [ \sum_{s=1}^d w(s)^2 (v v^\top x  x^\top)(s, s)]$ term. We have
\begin{align} \label{eq:expectation_sum_w_vx_diagonal}
    \E_{w} [ \sum_{s=1}^d w(s)^2 (v v^\top x  x^\top)(s, s)]
    = & ~ \sum_{s=1}^d (v v^\top x  x^\top)(s, s) \E_{w} [ w(s)^2] \notag \\
    = & ~ \sigma^2 \sum_{s=1}^d (v v^\top x  x^\top)(s, s)\notag \\
    = & ~ \sigma^2 (v^\top x) \sum_{s=1}^d (v   x^\top)(s, s)\notag \\
    = & ~ \sigma^2 (v^\top x)^2
\end{align}
where the 1st step is because of the property of expectation, the 2nd step is due to $\E_{w} [w_i^2] = \sigma^2$, the 3rd step is from $v^\top x \in \mathbb{R}$, the 4th step comes from $\sum_{s=1}^d (v \cdot   x^\top)(s, s) = v^\top x$.

Further, we can calculate
\begin{align*}
    & ~ | \E_{w} [ \sum_{s=1}^d w(s)^2 (v v^\top  x  x^\top -  x' x'^\top)(s, s)] | \\
    = & ~| \E_{w} [ \sum_{s=1}^d w(s)^2 (v v^\top  x  x^\top )(s, s)] - \E_{w} [ \sum_{s=1}^d w(s)^2 (v v^\top x' x'^\top)(s, s)] |  \\
    = & ~ \sigma^2 | (v^\top x)^2 - (v^\top x')^2 |
\end{align*}
where the 1st step is because of linearity of expectation, the 2nd step is due to Eq.~\eqref{eq:expectation_sum_w_vx_diagonal}.

Further, we can calculate 
\begin{align*}
    \sigma^2 | (v^\top x)^2 - (v^\top x')^2 | 
    = & ~  \sigma^2 | v^\top (x +  x') | \cdot | v^\top ( x -  x') | \\
    \leq & ~ \sigma^2 \| v \|_2 \cdot \| x + x' \|_2 \cdot  \| v \|_2 \cdot \| x - x' \|_2 \\
    \leq & ~ \sigma^2 B \cdot 2B \cdot  B \cdot \| x - x' \|_2 \\
    = & ~ 2 \sigma^2 B^3 \cdot \| x - x' \|_2 
\end{align*}
where the 1st step is because of $|a^2 - b^2| = |a + b| \cdot |a - b|$ for any $a, b \in \mathbb{R}$, the 2nd step is due to $| v^\top (x +  x') | \leq \| v \|_2 \cdot \| x + x' \|_2 $, the 3rd step is from the property of inner product, the 4th step comes from $\| x \|_2 \leq B$, the fifth step follows from basic algebra. 

Combining all components analysed above, we have
\begin{align*}
    | H^{\mathrm{cts}}_{i,j} - {H^{\mathrm{cts}}_{i,j}}' | \leq 2 \sigma^2 B^3 \cdot \| x - x' \|_2
\end{align*}

\end{proof}

Similar to off-diagonal entries, we consider the diagonal entries case. 

\begin{lemma} [Lipschitz of single entry (diagonal entries)] \label{lem:quadratic_ntk_diagonal_single_lipschitz}
If we have the below conditions,
\begin{itemize}
    \item Let $B > 0$ be a constant.
    \item Let $n$ be the number of data points.
    \item Let $m$ be the number of neurons. 
    \item Let dataset ${\cal D} = (X, Y)$, where $X \in \R^{n \times d}$ and $Y \in \R^n$. 
    \item Let $x_i \in \R^d$ denotes $X(i, *)$, and $\| x_i \|_2 \leq B$, for any $i \in [n]$. 
    \item Let ${\cal D'}$ be the neighbor dataset. (See Definition~\ref{def:beta_ntk_neighbor_dataset})
    Without loss of generality, we assume ${\cal D}$ and ${\cal D'}$ only differ in the $n$-th item. 
    \item Let $x := x_n \in {\cal D}$ and  $x' := x_n' \in {\cal D'}$. 
    \item Let $H^{\mathrm{cts}}$ and ${H^{\mathrm{cts}}}'$ be defined as Definition~\ref{def:cts_quadratic_ntk}, where $H_{i, j}^{\mathrm{cts}} =  \E_{w} \langle \langle w , x_i \rangle   x_i,  \langle w , x_j \rangle  x_j \rangle$. And 
    ${H^{\mathrm{cts}}}' \in \R^{n \times n}$ is the kernel corresponding to ${\cal D'}$.
\end{itemize}

Then, we can show that
\begin{align*}
    | H^{\mathrm{cts}}_{n, n} - {H^{\mathrm{cts}}_{n, n}}' | \leq 4 \sigma^2 B^3 \cdot  \| x - x' \|_2
\end{align*}

\end{lemma}

\begin{proof}

We have the following
\begin{align*}
    | H^{\mathrm{cts}}_{n, n} - {H^{\mathrm{cts}}_{n, n}}' | = & ~ |\E_{w} [\langle \langle w , x \rangle  x,  \langle w , x \rangle  x \rangle] - \E_{w} [\langle \langle w , x' \rangle  x',  \langle w , x' \rangle  x' \rangle]| \\
    = & ~ |\E_{w} [ \langle \langle w , x \rangle  x,  \langle w , x \rangle  x \rangle - \langle \langle w , x' \rangle  x',  \langle w , x' \rangle  x' \rangle] |
\end{align*}
where the 1st step is because of the definition of $H^{\mathrm{cts}}_{i, j}$, the 2nd step is due to linearity of expectation.

Then, we have
\begin{align*}
    | \E_{w} [ \langle \langle w , x \rangle  x,  \langle w , x \rangle  x \rangle - \langle \langle w , x' \rangle  x',  \langle w , x' \rangle  x' \rangle ] | 
    = & ~ | \E_{w} [ w^\top x  x^\top x x^\top w -   w^\top x'  x'^\top x' x'^\top w ] | \\
    = & ~ | \E_{w} [ w^\top( x  x^\top x x^\top - x'  x'^\top x' x'^\top) w ] |
\end{align*}
where the 1st step is because of
Lemma~\ref{lem:nested_inner_prod_property}
, the 2nd step is due to the basic linear algebra. 

Let $A := (x  x^\top x x^\top   - x'  x'^\top x' x'^\top) \in \mathbb{R}^{d \times d}$. Let $A(s, t) \in \mathbb{R}$ denotes the $(s, t)$-th entry of $A$.

Let $w(s) \in \mathbb{R}$ denotes the the $s$-th entry of $w \in \mathbb{R}^d$.

Further, we can calculate
\begin{align*}
    | \E_{w} [ w^\top( x  x^\top x x^\top - x'  x'^\top x' x'^\top) w ] | 
    = & ~ \E_{w} [ w^\top A w ] \\
    = & ~ | \E_{w} [ \sum_{s=1}^d w(s)^2 A(s, s) ] + \E_{w} [ \sum_{s \neq t} w(s) w(t) A(s, t) ] | \\
    = & ~ | \E_{w} [\sum_{s=1}^d w(s)^2 A(s, s) ] | \\
    = & ~ | \E_{w} [ \sum_{s=1}^d w(s)^2 (x  x^\top x x^\top   - x'  x'^\top x' x'^\top)(s, s) ] | 
\end{align*}
where the 1st step is because of definition of $A$, the 2nd step is due to we separate the diagonal and off-diagonal term, the 3rd step is from $\E[w(s) w(t)] = \E[w(s)] \E[w(t)] =0$ (when $w(s)$ and $w(t)$ are independent), the 4th step comes from definition of $A$. 

Let $x(s) \in \mathbb{R}$ denotes the $s$-th entry of $x \in \mathbb{R}^d$. Let $x(s)' \in \mathbb{R}$ denotes the $s$-th entry of $x' \in \mathbb{R}^d$. 

Then, we consider the $\E_{w} [ \sum_{s=1}^d w(s)^2 (x  x^\top x x^\top)(s, s) ]$ term. We have
\begin{align} \label{eq:expectation_sum_w_diagonal_entry}
    \E_{w} [ \sum_{s=1}^d w(s)^2 (x  x^\top x x^\top)(s, s) ] 
    = & ~ \sum_{s=1}^d  (x  x^\top x x^\top)(s, s) \E_{w} [ w(s)^2 ] \notag \\
    = & ~ \sigma^2 \sum_{s=1}^d  (x  x^\top x x^\top)(s, s) \notag \\
    = & ~ \sigma^2 ( x^\top x ) \sum_{s=1}^d  (x  x^\top)(s, s) \notag \\
    = & ~ \sigma^2 ( x^\top x )^2
\end{align}
where the 1st step is because of linearity of expectation, the 2nd step is due to $\E_{w} [w(s)^2] = \sigma^2$, the 3rd step is from $x^\top x \in \mathbb{R}$, the 4th step comes from $\sum_{s=1}^d (x \cdot   x^\top)(s, s) = x^\top x$. 

Further, we can calculate
\begin{align*}
    & ~ | \E_{w} [ \sum_{s=1}^d w(s)^2 (x  x^\top x x^\top   - x'  x'^\top x' x'^\top)(s, s) ] | \\
    = & ~ | \E_{w} [ \sum_{s=1}^d w(s)^2 (x  x^\top x x^\top)(s, s) ] - \E_{w} [ \sum_{s=1}^d w(s)^2 ( x'  x'^\top x' x'^\top)(s, s) ] | \\
    = & ~ \sigma^2 | ( x^\top x )^2 - ( x'^\top x' )^2 | \\
    = & ~ \sigma^2 | \| x \|_2^4 - \| x' \|_2^4  |
\end{align*}
where the 1st step is because of linearity of expectation, the 2nd step is due to Eq.~\eqref{eq:expectation_sum_w_diagonal_entry}, the 3rd step is from $x^\top x = \| 
x \|_2^2$. 

Further, we can calculate
\begin{align*}
     \sigma^2 | \| x \|_2^4 - \| x' \|_2^4  | 
     = & ~ \sigma^2 |  \| x \|_2^2 + \| x' \|_2^2 | \cdot |  \| x \|_2 + \| x' \|_2 | \cdot |  \| x \|_2 - \| x' \|_2 | \\
     \leq & ~ \sigma^2 |  \| x \|_2^2 + \| x' \|_2^2 | \cdot |  \| x \|_2 + \| x' \|_2 | \cdot  \| x - x' \|_2  \\
    \leq & ~ \sigma^2 \cdot 2B^2 \cdot 2B \cdot  \| x - x' \|_2  \\
    = & ~ 4 \sigma^2 B^3 \cdot  \| x - x' \|_2
\end{align*}
where the 1st step is because of $|a^2 - b^2| = |a + b| \cdot |a - b|$ for any $a, b \in \mathbb{R}$, the 2nd step is due to triangle inequality $|  \| x \|_2 - \| x' \|_2 | \leq \| x - x' \|_2$, the 3rd step is from $\| x_i \|_2 \leq B$, the 4th step comes from basic algebra.

Combining all components analysed above, we have
\begin{align*}
    | H^{\mathrm{cts}}_{n, n} - {H^{\mathrm{cts}}_{n, n}}' | \leq 4 \sigma^2 B^3 \cdot  \| x - x' \|_2
\end{align*}

\end{proof}

Having established the Lipschitz properties for both the diagonal and off-diagonal entries, we are now equipped to demonstrate the sensitivity of the Continuous Quadratic NTK Matrix with respect to $\beta$-close neighboring datasets.

\begin{lemma} [Sensitivity of Continuous Quadratic NTK under $\beta$-close neighboring dataset ] \label{lem:beta_cts_quadratic_sensitivity}
If we have the below conditions,
\begin{itemize}
    \item Let $B > 0$ be a constant.
    \item Let $n$ be the number of data points.
    \item We have dataset ${\cal D} = \{(x_i, y_i)\}_{i=1}^n$, where $x_i \in \R^d$ and $\|x_i\|_2 \le B$ for any $i \in [n]$.
    \item Let the neighbor dataset ${\cal D'}$ be defined Definition~\ref{def:beta_ntk_neighbor_dataset}. 
    \item Let the continuous quadratic kernel matrices $H^{\mathrm{cts}}$ and ${H^{\mathrm{cts}}}'$ be defined as Definition~\ref{def:cts_quadratic_ntk}, where ${H^{\mathrm{cts}}}' \in \R^{n \times n}$ is the kernel corresponding to ${\cal D'}$. Without loss of generality, we have ${\cal D}$ and ${\cal D'}$ only differ in the $n$-th item. 
\end{itemize}

Then, we can show that
\begin{align*}
    \| H^{\mathrm{cts}} - {H^{\mathrm{cts}}}' \|_F \leq O( \sqrt{n} \sigma^2 B^3 \beta)
\end{align*}
\end{lemma}

\begin{proof}
By Lemma~\ref{lem:quadratic_ntk_non_diagonal_single_lipschitz}, we know that, for all $\{ 
(i, j) : (i = n, j \neq n) or (i \neq n, j = n), i, j \in [n] \}$, we have 
\begin{align} \label{eq:beta_off_diagonal_lipschitz}
    | H^{\mathrm{cts}}_{i,j} - {H^{\mathrm{cts}}_{i,j}}' | \leq 2 \sigma^2 B^3 \cdot \| x - x' \|_2
\end{align}

By Lemma~\ref{lem:quadratic_ntk_diagonal_single_lipschitz}, we know that
\begin{align} \label{eq:beta_diagonal_lipschitz}
    | H^{\mathrm{cts}}_{n, n} - {H^{\mathrm{cts}}_{n, n}}' | \leq 4 \sigma^2 B^3 \cdot  \| x - x' \|_2
\end{align}

By Definition~\ref{def:beta_ntk_neighbor_dataset}, we have
\begin{align} \label{eq:x_diff_x_prim_with_beta}
    \| x - x' \|_2 \leq \beta
\end{align}

There are $2n - 2$ entries satisfy  $\{ 
(i, j) : (i = n, j \neq n) or (i \neq n, j = n), i, j \in [n] \}$. And there is $1$ entry satisfies $(i, j) = (n, n)$.

Therefore, we have
\begin{align*}
    \| H^{\mathrm{cts}} - {H^{\mathrm{cts}}}' \|^2_F 
    = & ~ (2n - 2) ( H^{\mathrm{cts}}_{i,j} - {H^{\mathrm{cts}}_{i,j}}')^2 + ( H^{\mathrm{cts}}_{n, n} - {H^{\mathrm{cts}}_{n, n}}')^2  \\
    \leq & ~ (2n - 2) \cdot 4 \sigma^4 B^6 \| x - x' \|_2^2 + 16 \sigma^4 B^6 \| x - x' \|_2^2 \\
    \leq & ~ (2n - 2) \cdot 4 \sigma^4 B^6 \cdot \beta^2 + 16 \sigma^4 B^6 \cdot \beta^2 \\
    = & ~ O(n\sigma^4 B^6 \beta^2)
\end{align*}
where the 1st step is because of definition of $\| \cdot \|_F$, the 2nd step is due to Eq.~\eqref{eq:beta_off_diagonal_lipschitz} and Eq.~\eqref{eq:beta_diagonal_lipschitz}, the 3rd step is from Eq.~\eqref{eq:x_diff_x_prim_with_beta}, the 4th step comes from basic algebra.  

Therefore, we have 
\begin{align*}
    \| H^{\mathrm{cts}} - {H^{\mathrm{cts}}}' \|_F \leq O( \sqrt{n} \sigma^2 B^3 \beta)
\end{align*}

\end{proof}

\subsection{Sensitivity of Discrete Quadratic NTK} \label{sec:dis_ntk_kernel_sensitivity}

After calculating the sensitivity of the Continuous Quadratic NTK Matrix in the previous section, we are going to use Bernstein Inequality to bridge the gap between Continuous and Discrete Quadratic NTK Matrices, where eventually we are going to demonstrate the sensitivity of Discrete Quadratic NTK under $\beta$-close neighboring dataset. 

We first introduce a fundamental and useful property of a single entry of the Continuous Quadratic NTK Matrix. 

\begin{lemma} [Close form of $H^{\mathrm{cts}}_{i,j}$] \label{lem:close_form_of_cts_kernel}
If we have the below conditions,
\begin{itemize}
    \item Let $H^{\mathrm{dis}}$ be define in Definition~\ref{def:dis_quadratic_ntk}. Namely, $H_{i, j}^{\mathrm{dis}} = \frac{1}{m} \sum_{r=1}^m \langle \langle w_r , x_i \rangle   x_i,  \langle w_r , x_j \rangle  x_j \rangle$.
    \item Let $H^{\mathrm{cts}}$ be define in Definition~\ref{def:cts_quadratic_ntk}. Namely, $H_{i, j}^{\mathrm{cts}} = \E_{w \sim {\cal N}(0, \sigma^2 I_{d\times d})} \langle \langle w , x_i \rangle   x_i,  \langle w , x_j \rangle  x_j \rangle$.
\end{itemize}

Then we have 
\begin{align*}
    \E_{w} [H_{i, j}^{\mathrm{dis}}] = & ~ H_{i, j}^{\mathrm{cts}} \\
    H_{i, j}^{\mathrm{cts}} = & ~ \sigma^2 (x_i^\top x_j)^2
\end{align*}
\end{lemma}

\begin{proof}
We have the following equation
\begin{align*}
    \E_{w} [H_{i, j}^{\mathrm{dis}}] = & ~ \E_{w} [\frac{1}{m} \sum_{r=1}^m \langle \langle w_r , x_i \rangle   x_i,  \langle w_r , x_j \rangle  x_j \rangle] \\
    = & ~ \frac{1}{m} \sum_{r=1}^m \E_{w_r} [ \langle \langle w_r , x_i \rangle   x_i,  \langle w_r , x_j \rangle  x_j \rangle] \\
    = & ~ \E_{w} [ \langle \langle w , x_i \rangle   x_i,  \langle w , x_j \rangle  x_j \rangle] \\
    = & ~ H_{i, j}^{\mathrm{cts}}
\end{align*}
where the 1st step is because of the definition of $H_{i, j}^{\mathrm{dis}}$, the 2nd step is due to the linearity of expectation, the 3rd step is from $w_r$ are i.i.d. , the 4th step comes from the definition of $H_{i, j}^{\mathrm{cts}}$. 

\begin{align*}
    H_{i, j}^{\mathrm{cts}} \leq \sigma^2 B^4
\end{align*}

Then, we can calculate the following
\begin{align*}
    H_{i, j}^{\mathrm{cts}} 
    = & ~ \E_{w} [\langle \langle w, x_i \rangle x_i, \langle w, x_j \rangle x_j \rangle] \\
    = & ~ \E_{w}[ w^\top x_i x_i^\top x_j x_j^\top w ] | \\
    = & ~ \E_{w}[ \sum_{s=1}^d w(s)^2 (x_i x_i^\top x_j x_j^\top)(s, s) ] + \E_{w}[ \sum_{s \neq t} w(s)w(t) (x_i x_i^\top x_j x_j^\top)(s, t) ] \\
    = & ~ \E_{w}[ \sum_{s=1}^d w(s)^2 (x_i x_i^\top x_j x_j^\top)(s, s) ] \\
    = & ~ \sigma^2 \sum_{s=1}^d (x_i x_i^\top x_j x_j^\top)(s, s) \\
    = & ~ \sigma^2 \cdot \langle x_i, x_j \rangle  \sum_{s=1}^d (x_i  x_j^\top)(s, s) \\
     = & ~ \sigma^2 (\langle x_i, x_j \rangle)^2
\end{align*}
where the 1st step is because of the definition of $H_{i, j}^{\mathrm{cts}}$, the 2nd step is due to Lemma~\ref{lem:nested_inner_prod_property}, the 3rd step is from we separate the diagonal term and off-diagonal term, the 4th step comes from $\E[w(s) w(t)] = \E[w(s)] \E[w(t)] =0$ (when $w(s)$ and $w(t)$ are independent), the fifth step follows from $\E_{w} [w(s)^2] = \sigma^2$, the sixth step follows from $x_i^\top x_j \in \mathbb{R}$, the seventh step follows from $\sum_{s=1}^d (x_i  x_j^\top)(s, s) = x_i^\top x_j \in \mathbb{R}$. 

\end{proof}

In order to use Bernstein Inequality, we need to bound the value ranges of random variables,as well as the variance of random variables. 

We first consider the value range of $\| w_r \|_2^2$. 

\begin{lemma} [Upper bound $\| w_r \|_2^2$ ] \label{lem:upper_bound_w_r}
If we have the below conditions,
\begin{itemize}
    \item Let $w_r \sim {\cal N}(0, \sigma^2 I_{ d\times d})$.
    \item Let $\delta_1 \in (0, 1)$.
    \item Let $t \geq \Omega( \log (1 / \delta_1))$
\end{itemize}
Then, with probability $1 - \delta_1$, we have
\begin{align*}
    \| w_r \|_2^2 \leq 3(t + d) \sigma^2
\end{align*}
\end{lemma}

\begin{proof}
    Let $w_r(s) \in \mathbb{R}$ denotes the $s$-th entry of $w_r \in \mathbb{R}^d$.

We consider $\| w_r \|_2^2$. We have
\begin{align*}
    \| w_r \|_2^2 = \sum_{s=1}^d w_r(s)^2
\end{align*}

Since for each $s \in [d]$, $w_r(s) \sim {\cal N}(0, \sigma)$, then $\| w_r \|_2^2 \sim \chi^2_d$, where each one has zero means and $\sigma^2$ variance. 

By Chi-Square tail bound (Lemma~\ref{lem:chi_square_bound}), we have
\begin{align} \label{eq:original_chi_square_bound_in_w}
     \Pr[\| w_r \|_2^2 - d\sigma^2 \geq (2\sqrt{dt} + 2t) \sigma^2]
    \leq & ~ \exp{(-t)}
\end{align}

Our goal is to choose $t$ sufficiently large, such that the above quantity is upper bounded by $\delta_1$. 

We choose $t \geq \Omega (\log ( m n / \delta_1))$, and substitute $t$ in Eq.~\eqref{eq:original_chi_square_bound_in_w}. Then we have
\begin{align*}
    \Pr[\| w_r \|_2^2 \geq (2\sqrt{dt} + 2t + d) \sigma^2]
    \leq & ~ \delta_1 / \poly (m, n)
\end{align*}

which is equivalent to 
\begin{align*}
    \Pr[\| w_r \|_2^2 \leq (2\sqrt{dt} + 2t + d) \sigma^2]
    \geq & ~ 1 -\delta_1 / \poly (m, n)
\end{align*}

Since $(2\sqrt{dt} + 2t + d) \leq 2(\sqrt{dt} + t + d) \leq 3(t + d)$, then we have
\begin{align*}
    \Pr[\| w_r \|_2^2 \leq 3(t + d) \sigma^2]
    \geq & ~ 1 -\delta_1 / \poly (m)
\end{align*}

\end{proof}

After calculating the value range of $\| w_r \|_2^2$, we are ready to prove the value range of random variables $Z_r$. 

\begin{lemma} [Upper bound $Z_r$ ] \label{lem:upper_bound_z_i}
If we have the below conditions,
\begin{itemize}
    \item Let $H^{\mathrm{dis}} \in \mathbb{R}^{n \times n}$ be define in Definition~\ref{def:dis_quadratic_ntk}. Namely, $H_{i, j}^{\mathrm{dis}} = \frac{1}{m} \sum_{r=1}^m \langle \langle w_r , x_i \rangle   x_i,  \langle w_r , x_j \rangle  x_j \rangle$, for any $i, j \in [n]$. 
    \item Let $\Xi_r := \langle \langle w_r , x_i \rangle   x_i,  \langle w_r , x_j \rangle  x_j \rangle$, where $\Xi_r \in \mathbb{R}$.
    \item Let $Z_r := \Xi_r - \E_{w_r} [\Xi_r]$, where $\E [Z_r] = 0$. 
    \item Let $\delta_1 \in (0, 1)$
    \item Let $d \geq \Omega (\log (1 / \delta_1))$
\end{itemize}

Then, with probability $1 - \delta_1$, we have
\begin{itemize}
    \item {\bf Part 1.} $|\Xi_r| \leq 6 d \sigma^2 B^4$
    \item {\bf Part 2.} $|Z_r| \leq 6 d \sigma^2 B^4$
\end{itemize}

\end{lemma}

\begin{proof}
{\bf Proof of Part 1.}
For any $r \in [m]$, we have
\begin{align*}
    \Xi_r = & ~ \langle w_r , x_i \rangle  \langle w_r , x_j \rangle \langle    x_i,   x_j \rangle \\
    \leq & ~ \| w_r \|_2^2 \cdot \| x_i \|_2^2 \cdot \| x_j \|_2^2 \\
    \leq & ~ \| w_r \|_2^2 \cdot B^4 
\end{align*}
where the 1st step is because of $\langle a, b \rangle \leq \|a \|_2 \|b \|_2$ for any $a, b \in \R^d$ holds, the 2nd step is due to $\| x_i \|_2 \leq B$.  

By Lemma~\ref{lem:upper_bound_w_r},for any $\delta_1 \in (0, 1)$, if $t \geq \Omega(\log (1 / \delta_1))$ we have the following holds
\begin{align*}
    \Pr[\| w_r \|_2^2 \leq 3 (t + d) \sigma^2 ]
    \geq & ~ 1 - \delta_1 
\end{align*}

{\bf Proof of Part 2.}
Then, we have
\begin{align*}
    \Pr[ \Xi_r \leq 3( t + d) \sigma^2 B^4 ]
    \geq & ~ 1 - \delta_1
\end{align*}

Since $\Xi_r = \langle w_r , x_i \rangle  \langle w_r , x_j \rangle \langle    x_i,   x_j \rangle$, and $\langle a , b \rangle \geq 0$ holds for any $a, b \in \mathbb{R}^d$, then we have 
\begin{align*}
    \Xi_r \geq & ~ 0 \\
    \E [\Xi_r] \geq & ~ 0
\end{align*}

Then, we consider $|Z_r|$. We have
\begin{align*}
    |Z_r| = & ~ |\Xi_r - \E [\Xi_r]| \\
    \leq & ~ |\Xi_r|
\end{align*}
where the 1st step is because of the definition of $Z_r$, the 2nd step is due to $\Xi_r \geq 0$.

Combining all analyses above, we have
\begin{align*}
    \Pr[ |Z_r| \leq 3 (t + d) \sigma^2 B^4 ]
    \geq & ~ 1 - \delta_1 
\end{align*}

We choose $d \geq \Omega (\log (m n / \delta_1))$. We have
\begin{align*}
    \Pr [ |Z_r| \leq 6 d \sigma^2 B^4 ] \geq 1 - \delta_1 
\end{align*}

\end{proof}

So far, we have bounded the value range of random variables. We are going to bound the variance of random variables. Since we have the Expectation of random variables are $0$, we only need to consider the second moment of random variables. 

\begin{lemma} [Second moment of $\Xi_r$] \label{lem:second_moment_of_Xi_r}
    If we have the below conditions,
\begin{itemize}
    \item Let $H^{\mathrm{dis}} \in \mathbb{R}^{n \times n}$ be define in Definition~\ref{def:dis_quadratic_ntk}. Namely, $H_{i, j}^{\mathrm{dis}} = \frac{1}{m} \sum_{r=1}^m \langle \langle w_r , x_i \rangle   x_i,  \langle w_r , x_j \rangle  x_j \rangle$, for any $i, j \in [n]$
    \item Let $\Xi_r := \langle \langle w_r , x_i \rangle   x_i,  \langle w_r , x_j \rangle  x_j \rangle$, where $\Xi_r \in \mathbb{R}$.
\end{itemize}

Then we have
\begin{align*}
    \E_{w_r} [\Xi_r^2] = (\sum_{s=1}^d x_i(s) x_j(s))^2 \cdot (3 \sigma^2 (\sum_{s=1}^d x_i(s)^2 x_j(s)^2) + 2 \sigma^4 (\sum_{s \neq t} x_i(s)^2 x_j(t)^2 ))
\end{align*}

\end{lemma}

\begin{proof}
Let $v(s) \in \mathbb{R}$ denotes the $s$-th entry of any $v \in \mathbb{R}^d$.  

Then, we have
\begin{align} \label{eq:y_r_eq_three_summations}
    \Xi_r = & ~ \langle \langle w_r , x_i \rangle   x_i,  \langle w_r , x_j \rangle  x_j \rangle \notag \\
    = & ~ w_r^\top x_i x_i^\top x_j x_j^\top w \notag \\
    = & ~ (\sum_{s=1}^d x_i(s) w_r(s)) \cdot (\sum_{s=1}^d x_i(s) x_j(s)) \cdot (\sum_{s=1}^d x_j(s) w_r(s))
\end{align}
where the 1st step is because of the definition of $\Xi_r$, the 2nd step is due to Lemma~\ref{lem:nested_inner_prod_property}, the 3rd step is from $\langle a, b\rangle = \sum_{s=1}^d a(s) b(s)$ holds for any $a, b \in \mathbb{R}^d$. 

Then, we have
\begin{align} \label{eq:origin_huge_Xi_r_square}
    \Xi_r^2 = & ~ (\sum_{s=1}^d x_i(s) w_r(s))^2 \cdot (\sum_{s=1}^d x_i(s) x_j(s))^2 \cdot (\sum_{s=1}^d x_j(s) w_r(s))^2 \notag \\
    = & ~ (\sum_{s=1}^d x_i(s) x_j(s))^2 \notag \\
    & ~ \cdot (\sum_{s=1}^d x_i(s)^2 w_r(s)^2 + \sum_{s \neq t} x_i(s) x_i(t) w_r(s) w_r(t) ) \notag \\
    & ~ \cdot (\sum_{s=1}^d x_j(s)^2 w_r(s)^2 + \sum_{s \neq t} x_j(s) x_j(t) w_r(s) w_r(t) )
\end{align}
where the 1st step is because of Eq.~\eqref{eq:y_r_eq_three_summations}, the 2nd step is due to we separate the diagonal and off-diagonal term. 

We reorganize Eq.~\eqref{eq:origin_huge_Xi_r_square}, into the following format
\begin{align*}
    \Xi_r^2 = C_0 (T_1 + T_2 + T_3 + T_4) 
\end{align*}
where
\begin{align*}
    C_0 = & ~ (\sum_{s=1}^d x_i(s) x_j(s))^2 \\
    T_1 = & ~ (\sum_{s=1}^d x_i(s)^2 w_r(s)^2) (\sum_{s=1}^d x_j(s)^2 w_r(s)^2) \\
    T_2 = & ~ (\sum_{s=1}^d x_i(s)^2 w_r(s)^2) (\sum_{s \neq t} x_j(s) x_j(t) w_r(s) w_r(t)) \\
    T_3 = & ~ (\sum_{s \neq t} x_i(s) x_i(t) w_r(s) w_r(t)) (\sum_{s=1}^d x_j(s)^2 w_r(s)^2) \\
    T_4 = & ~ (\sum_{s \neq t} x_i(s) x_i(t) w_r(s) w_r(t)) (\sum_{s \neq t} x_j(s) x_j(t) w_r(s) w_r(t)) 
\end{align*}

We consider $T_1$. We have
\begin{align*}
    T_1 = & ~ \sum_{s=1}^d x_i(s)^2 x_j(s)^2 w_r(s)^4 + \sum_{s \neq j} x_i(s)^2 x_j(t)^2 w_r(s)^2 w_r(t)^2 
\end{align*}

Then, we have
\begin{align*}
    \E [T_1] = & ~ 3 \sigma^2 (\sum_{s=1}^d x_i(s)^2 x_j(s)^2) + \E [\sum_{s \neq j} x_i(s)^2 x_j(t)^2 w_r(s)^2 w_r(t)^2] \\
    = & ~ 3 \sigma^2 (\sum_{s=1}^d x_i(s)^2 x_j(s)^2) + \sigma^4 (\sum_{s \neq j} x_i(s)^2 x_j(t)^2 )
\end{align*}
where the 1st step is because of $\E [w_r(s)^4] = 2\sigma^2$, the 2nd step is due to $\E [W_r(s)^2] = \sigma^2$. 

We consider $T_2$. We know that for each single term in $T_2$, it must contain one of the following terms
\begin{itemize}
    \item $w_r(s)^3w_r(t)$
    \item $w_r(s)w_r(t)^3$
    \item $w_r(s)w_r(t)w_r(k)^2$
\end{itemize}

Since we have $\E [w_r(s)] = 0$, then we can have
\begin{align*}
    \E [T_2] = 0
\end{align*}

The same analysis can be applied to $T_3$. Then we have
\begin{align*}
    \E [T_3] = 0
\end{align*}

We consider $T_4$. We have
\begin{align*}
    \E [T_4] = & ~ \E [(\sum_{s \neq t} x_i(s) x_i(t) w_r(s) w_r(t)) (\sum_{s \neq t} x_j(s) x_j(t) w_r(s) w_r(t))] \\
    = & ~ \E [\sum_{s \neq t} x_i(s)^2 x_i(t)^2 w_r(s)^2 w_r(t)^2] \\
    = & ~ \sigma^4 \cdot (\sum_{s \neq t} x_i(s)^2 x_i(t)^2)
\end{align*}
where the 1st step is because of the definition of $T_4$, the 2nd step is due to similar analysis of $T_2$ and $T_3$, the 3rd step is from $\E [w_r(s)] = \sigma^2$. 

Then we consider $\Xi_r^2$. We have
\begin{align*}
    \E [\Xi_r^2] = & ~ C_0 (\E [T_1] + \E [T_2] + \E [T_3] + \E [T_4]) \\
    = & ~ (\sum_{s=1}^d x_i(s) x_j(s))^2 \cdot (3 \sigma^2 (\sum_{s=1}^d x_i(s)^2 x_j(s)^2) + 2 \sigma^4 (\sum_{s \neq j} x_i(s)^2 x_j(t)^2 ))
\end{align*}
where the 1st step is because of linearity of expectation, the 2nd step is due to above analysis about $\E[T_1], \E[T_2], \E[T_3], \E[T_4]$. 

\end{proof}

Since we have $Z_r := \Xi_r - \E_{w_r} [\Xi_r]$, then $\E [Z_r] = 0$. We can bound variance of $Z_r$ by harnessing the power of the bound of second moment of $Z_r$, which can be formally expressed as follows. 

\begin{lemma} [variance of $Z_r$] \label{lem:variance_of_z_i}
If we have the below conditions,
\begin{itemize}
    \item Let $H^{\mathrm{dis}} \in \mathbb{R}^{n \times n}$ be define in Definition~\ref{def:dis_quadratic_ntk}. Namely, $H_{i, j}^{\mathrm{dis}} = \frac{1}{m} \sum_{r=1}^m \langle \langle w_r , x_i \rangle   x_i,  \langle w_r , x_j \rangle  x_j \rangle$, for any $i, j \in [n]$. 
    \item Let $\Xi_r := \langle \langle w_r , x_i \rangle   x_i,  \langle w_r , x_j \rangle  x_j \rangle$, where $\Xi_r \in \mathbb{R}$.
    \item Let $Z_r := \Xi_r - \E_{w_r} [\Xi_r]$, where $\E [Z_r] = 0$. 
\end{itemize}

Then, we have
\begin{align*}
    \mathrm{Var} [Z_r] \leq 6 B^8 \sigma^4
\end{align*}

\end{lemma}

\begin{proof}
We consider $\E [Z_r^2]$. We have
\begin{align*}
    \E [Z_r^2] = & ~ \E [(\Xi_r - \E [\Xi_r])^2] \\
     = & ~ \E [\Xi_r^2 - 2 \E [\Xi_r] \Xi_r  + \E [\Xi_r]^2] \\
     = & ~ \E [\Xi_r^2] - \E [\Xi_r]
\end{align*}
where the 1st step is because of the definition of $Z_r$, the 2nd step is due to basic algebra, the 3rd step is from linearity of expectation. 

By Lemma~\ref{lem:close_form_of_cts_kernel}, we have
\begin{align*}
    \E [\Xi_r] = \sigma^2 (\sum_{s=1}x_i(s) x_j(s))^2
\end{align*}

By Lemma~\ref{lem:second_moment_of_Xi_r}, we have

\begin{align*}
    \E [\Xi_r^2] = (\sum_{s=1}^d x_i(s) x_j(s))^2 \cdot (3 \sigma^2 (\sum_{s=1}^d x_i(s)^2 x_j(s)^2) + 2 \sigma^4 (\sum_{s \neq j} x_i(s)^2 x_j(t)^2 ))
\end{align*}

Then, we have
\begin{align} \label{eq:second_moment_of_z_r}
    \E [Z_r^2] = & ~ \E [\Xi_r^2] - \E [\Xi_r] \notag \\
     = & ~ (\sum_{s=1}^d x_i(s) x_j(s))^2 \cdot (3 \sigma^2 (\sum_{s=1}^d x_i(s)^2 x_j(s)^2) + 2 \sigma^4 (\sum_{s \neq j} x_i(s)^2 x_j(t)^2 ) - \sigma^2)
\end{align}

Then, we have
\begin{align*}
    \mathrm{Var} [Z_r] = & ~ \E [Z_r^2] - (\E [Z_r])^2 \\
    = & ~ \E [Z_r^2] \\
    = & ~  (\sum_{s=1}^d x_i(s) x_j(s))^2 \cdot (3 \sigma^2 (\sum_{s=1}^d x_i(s)^2 x_j(s)^2) + 2 \sigma^4 (\sum_{s \neq j} x_i(s)^2 x_j(t)^2 ) - \sigma^2)
\end{align*}
where the 1st step is because of $\mathrm{Var} [X] = \E [X^2] - \E [X]^2$ holds for any random variable $X$, the 2nd step is due to $\E [Z_r] = 0$, the 3rd step is from Eq.~\eqref{eq:second_moment_of_z_r}. 

Then, we have
\begin{align*}
    \mathrm{Var} [Z_r] \leq & ~ 3 (\sum_{s=1}^d x_i(s) x_j(s))^2 \cdot ( \sigma^2 (\sum_{s=1}^d x_i(s)^2 x_j(s)^2) + \sigma^4 (\sum_{s \neq j} x_i(s)^2 x_j(t)^2 )) \\
    \leq & ~ 3 \| x_i \|_2^2 \cdot \| x_j \|_2^2 \cdot ( \sigma^2 (\sum_{s=1}^d x_i(s)^2 x_j(s)^2) + \sigma^4 (\sum_{s \neq j} x_i(s)^2 x_j(t)^2 )) \\
    \leq & ~ 3 B^4 \cdot ( \sigma^2 (\sum_{s=1}^d x_i(s)^2 x_j(s)^2) + \sigma^4 (\sum_{s \neq j} x_i(s)^2 x_j(t)^2 )) \\
    \leq & ~ 3 B^4 \cdot ( \sigma^2 (\sum_{s=1}^d x_i(s)^2 ) (\sum_{s=1}^d x_j(s)^2) + \sigma^4 (\sum_{s \neq j} x_i(s)^2 x_j(t)^2 )) \\
    \leq & ~ 3 B^4 \cdot ( \sigma^2 (\sum_{s=1}^d x_i(s)^2 ) (\sum_{s=1}^d x_j(s)^2) + \sigma^4 (\sum_{s=1}^d x_i(s)^2 ) (\sum_{s=1}^d x_j(s)^2)) \\
    \leq & ~ 3 B^4 \cdot ( \sigma^2 B^4 + \sigma^4 B^4) \\
    = & ~ 3 B^8 (\sigma^2 + \sigma^4)
\end{align*}
where the 1st step is because of basic inequality, the 2nd step is due to $x_i^\top x_j \leq \| x_i \|_2 \| x_j \|_2 $, the 3rd step is from $ \| x_i \|_2 \leq B$, the 4th step comes from $\sum_{s=1}^d x_i(s)^2 x_j(s)^2 \leq (\sum_{s=1}^d x_i(s)^2 ) (\sum_{s=1}^d x_j(s)^2)$, the fifth step follows from $\sum_{s \neq j} x_i(s)^2 x_j(t)^2 \leq (\sum_{s=1}^d x_i(s)^2 ) (\sum_{s=1}^d x_j(s)^2)$, the sixth step follows from $ \| x_i \|_2 \leq B$, the seventh step follows from basic algebra. 

We assume $\delta^2 \geq 1$. Then we have
\begin{align*}
    \mathrm{Var} [Z_r] \leq 6 B^8 \sigma^4
\end{align*}

\end{proof}

Since we have bounded the value range and variance of $Z_r$, we are ready to adapt Bernstein inequality to each entry of the Discrete Quadratic NTK Matrix. Therefore, we have the following Lemma. 

\begin{lemma} [Bernstein Inequality for $|H_{i, j}^{\mathrm{dis}} - H_{i, j}^{\mathrm{cts}}|$  ] \label{lem:bernstein_for_h_dis_ij_and_h_cts_ij}
If we have the below conditions,
\begin{itemize}
    \item Let $H^{\mathrm{dis}} \in \mathbb{R}^{n \times n}$ be define in Definition~\ref{def:dis_quadratic_ntk}. Namely, $H_{i, j}^{\mathrm{dis}} = \frac{1}{m} \sum_{r=1}^m \langle \langle w_r , x_i \rangle   x_i,  \langle w_r , x_j \rangle  x_j \rangle$, for any $i, j \in [n]$. 
    \item Let $H^{\mathrm{cts}}\in \mathbb{R}^{n \times n}$ be define in Definition~\ref{def:cts_quadratic_ntk}. Namely, $H_{i, j}^{\mathrm{cts}} = \E_{w \sim {\cal N}(0, \sigma^2 I_{d\times d})} \langle \langle w , x_i \rangle   x_i,  \langle w , x_j \rangle  x_j \rangle$, for any $i, j \in [n]$. 
    \item Let $\Xi_r := \langle \langle w_r , x_i \rangle   x_i,  \langle w_r , x_j \rangle  x_j \rangle$, where $\Xi_r \in \mathbb{R}$.
    \item Let $Z_r := \Xi_r - \E_{w_r} [\Xi_r]$, where $\E [Z_r] = 0$. 
    \item Let $\delta_1, \delta_2 \in (0, 1)$. Let $\delta_1 = \delta_2 / \poly(m)$. 
    \item Let $d = \Omega (\log (1 / \delta_1))$ 
    \item Let $m = \Omega ({\epsilon_1}^{-2} d B^8 \sigma^4  \log (1 / \delta_2)) $. 
\end{itemize}
 
Then, with probability $1 - \delta_2$, we have
\begin{align*}
   |H_{i, j}^{\mathrm{dis}} - H_{i, j}^{\mathrm{cts}}| \leq {\epsilon_1}
\end{align*}

\end{lemma}

\begin{proof}
Let $X_r := \frac{1}{m} Z_r$. Then we have $\E [X_r] = 0$.

Let $X := \sum_{r=1}^m X_r$. 

By Lemma~\ref{lem:upper_bound_z_i}, with probability $1 - \delta_1$, we have
\begin{align*}
    |Z_r| \leq 6 d \sigma^2 B^4
\end{align*}

which implies 

\begin{align} \label{eq:upper_bound_z_i}
    |X_r| \leq 6 d \sigma^2 B^4 m^{-1}\
\end{align}

By Lemma~\ref{lem:variance_of_z_i}, we have

\begin{align*}
    \mathrm{Var} [Z_r] \leq & ~ 6 B^8 \sigma^4
\end{align*}

which implies 
\begin{align} \label{eq:variance_of_z_i}
    \mathrm{Var} [X_r] \leq & ~  6 B^8 \sigma^4 m^{-2}
\end{align}

Applying Lemma~\ref{lem:scalar_bernstein}, we have
\begin{align} \label{eq:origin_x_bernstein}
    \Pr [ X > t ] \leq  \exp ( - \frac{ t^2/2 }{ m \E[X_r^2]  + M t /3 } )
\end{align}
where
\begin{align*}
    \E [X_r] = & ~ 0 \\
    \E [X_r^2] = & ~  6 B^8 \sigma^4 m^{-2} \\
    M = & ~ 6 d \sigma^2 B^4 m^{-1}
\end{align*}
where the first equation follows from the definition of $Z_r$, the second equation follows from Eq.~\eqref{eq:variance_of_z_i}, the third equation follows from Eq.~\eqref{eq:upper_bound_z_i}. 

Our goal is to choose $m$ sufficiently large such that Eq.~\eqref{eq:origin_x_bernstein} satisfies
\begin{align*}
    \Pr [ X > {\epsilon_1} ] \leq \delta_2 / 2
\end{align*}

We will achieve this in two steps. \begin{itemize}
    \item Firstly, we will find $t$ such that the probability is no larger than $\delta_2 / \poly(n) $. 
    \item Secondly, we will choose $m$ such that $t \leq {\epsilon_1}$. 
\end{itemize}

We need to choose $t$ such that Eq.~\eqref{eq:origin_x_bernstein} satisfies
\begin{align*}
    \Pr [ X > t ] \leq \delta_2 / 2
\end{align*}

Firstly, we need
\begin{align*}
    \frac{t^2}{m \E [X_r^2]} 
    = & ~ \frac{mt^2}{6 B^8 \sigma^4} \\
    \geq & ~ \log (1 / \delta_2)
\end{align*}

This requires
\begin{align*}
    t \geq (6 B^8 \sigma^4 m^{-1} \log (1 / \delta_2))^{1/2}
\end{align*}

Secondly, we need
\begin{align*}
    \frac{t^2}{Mt / 3} = & ~ \frac{3t}{M} \\
    = & ~  \frac{m t}{6 d \sigma^2 B^4} \\
    \geq & ~ \log (1 / \delta_2 )
\end{align*}

This requires 
\begin{align*}
    t \geq  6 d \sigma^2 B^4 m^{-1} \log (1 / \delta_2)
\end{align*}

We define
\begin{align*}
    A_1 = & ~ (6 B^8 \sigma^4 m^{-1} \log (1 / \delta_2))^{1/2} \\
    A_2 = & ~ 6 d \sigma^2 B^4 m^{-1} \log (1 / \delta_2)
\end{align*}

We should choose 
\begin{align*}
    t \geq A_1 + A_2
\end{align*}

Then, we need to choose $m$ sufficiently large, such that $t \leq {\epsilon_1}$. Namely, we need
\begin{align*}
    A_1 \leq & ~ {\epsilon_1} / 2 \\
    A_2 \leq & ~ {\epsilon_1} / 2
\end{align*}

Firstly, we consider $A_1$. We have
\begin{align*}
     A_1 = & ~ (6 B^8 \sigma^4 m^{-1} \log (1 / \delta_2))^{1/2} \\
     \leq & ~ {\epsilon_1} / 2
\end{align*}

This requires 
\begin{align*}
    m \geq C_1 B^8 \sigma^4 {\epsilon_1}^{-2} \log(1 / \delta_2)
\end{align*}

Secondly, we consider $A_2$. We have
\begin{align*}
    A_2 = & ~ 6 d \sigma^2 B^4 m^{-1} \log (1 / \delta_2) \\
    \leq & ~ {\epsilon_1} / 2
\end{align*}

This requires
\begin{align*}
    m \geq C_2 d \sigma^2 B^4 {\epsilon_1}^{-1} \log (1 / \delta_2 ) 
\end{align*}

Finally, we choose 
\begin{align} \label{eq:m_lower_bound}
    m \geq & ~ C_1 B^8 \sigma^4 {\epsilon_1}^{-2} \log(1 / \delta_2 ) + C_2 d \sigma^2 B^4 {\epsilon_1}^{-1} \log (1 / \delta_2 ) 
\end{align}

Therefore, 
By Eq.~\eqref{eq:m_lower_bound}, we choose $m = \Omega ({\epsilon_1}^{-2} d B^8 \sigma^4 \log (1 / \delta_2 ))$, then we have
\begin{align} \label{eq:bernstein_result_part_1}
    \Pr [ |X| \leq {\epsilon_1}] \geq 1 - \delta_2 / 2
\end{align}

We also need to consider the possibility that for each $r \in [m]$, $|X_r| \leq M$ holds. 

We choose $\delta_1 = \delta_2 / \poly(m)$ in Eq.~\eqref{eq:upper_bound_z_i}. 
Then We take union bound over $m$ entries, we have for all $r \in [m]$
\begin{align} \label{eq:bernstein_result_part_2}
    \Pr [ |X_r| \leq 6 d \sigma^2 B^4 m^{-1} ] 
    \geq & ~ 1 - m \delta_1 \notag \\
    = & ~ 1 - m \delta_2 / \poly(m) \notag \\
    \geq & ~ 1 - \delta_2 / 2
\end{align}

We take union bound over Eq.~\eqref{eq:bernstein_result_part_1} and Eq.~\eqref{eq:bernstein_result_part_2}. We have
\begin{align*}
    \Pr [ |H_{i, j}^{\mathrm{dis}} - H_{i, j}^{\mathrm{cts}}| \leq {\epsilon_1}] \geq & ~ 1 - (\delta_2 / 2 + \delta_2 / 2) \\
    = & ~ 1 - \delta_2
\end{align*}

\end{proof}

As we have already proven the gap between single entry between Continuous and Discrete Quadratic NTK Matrices in Lemma~\ref{lem:bernstein_for_h_dis_ij_and_h_cts_ij}, we further take union bound over all $n^2$ entries, to get the gap between Continuous and Discrete Quadratic NTK Matrices. 

\begin{lemma} [Union Bound $\| H^{\mathrm{dis}} - H^{\mathrm{cts}} \|_F$ 
] 
\label{lem:union_bound_of_h_dis_h_cts_f_norm}
If we have the below conditions,
\begin{itemize}
    \item Let $H^{\mathrm{dis}} \in \mathbb{R}^{n \times n}$ be define in Definition~\ref{def:dis_quadratic_ntk}. Namely, $H_{i, j}^{\mathrm{dis}} = \frac{1}{m} \sum_{r=1}^m \langle \langle w_r , x_i \rangle   x_i,  \langle w_r , x_j \rangle  x_j \rangle$, for any $i, j \in [n]$.
    \item Let $H^{\mathrm{cst}}\in \mathbb{R}^{n \times n}$ be define in Definition~\ref{def:cts_quadratic_ntk}. Namely, $H_{i, j}^{\mathrm{cts}} =  \E_{w \sim {\cal N}(0, \sigma^2 I_{d\times d})} \langle \langle w , x_i \rangle   x_i,  \langle w , x_j \rangle  x_j \rangle$, for any $i, j \in [n]$.
    \item Let $\delta_1, \delta_2 , \delta_3 \in (0, 1)$. Let $\delta_1 = \delta_2 / \poly(m)$. Let $\delta_2 = \delta_3 / \poly(n)$. 
    \item Let $d = \Omega (\log (1 / \delta_1))$ 
    \item Let $m = \Omega ( {\epsilon_1}^{-2} d B^8 \sigma^4 \log (1 / \delta_2))$
\end{itemize}

Then, with probability $1 - \delta_3$, we have
\begin{align*}
    \|H^{\mathrm{dis}} - H^{\mathrm{cts}}\|_F \leq n {\epsilon_1}
\end{align*}
\end{lemma}

\begin{proof}
By Lemma~\ref{lem:bernstein_for_h_dis_ij_and_h_cts_ij}, if we choose $m = \Omega ( {\epsilon_1}^{-2} d B^8 \sigma^4 \log (1 / \delta_2))$. Then, with probability $1 - \delta_2$ we have
\begin{align*}
    |H_{i, j}^{\mathrm{dis}} - H_{i, j}^{\mathrm{cts}}| \le {\epsilon_1} 
\end{align*}

We choose $\delta_2 = \delta_3 / \poly(n)$. Then, we take union bound over $n^2$ entries. We have
\begin{align*}
    \Pr [\|H^{\mathrm{dis}} - H^{\mathrm{cts}}\|_F \leq n \cdot {\epsilon_1} ] \geq & ~ 1 - n^2 \delta_2 \\
    = & ~ 1 - n^2 \delta_3 / \poly(n) \\
    \geq & ~ 1 - \delta_3
\end{align*}

\end{proof}

With the help of Bernstein Inequality and the sensitivity of Continuous Quadratic NTK Matrix, we are ready to present the sensitivity of Discrete Quadratic NTK Matrix. 

\begin{lemma} [Sensitivity of Discrete Quadratic NTK under $\beta$-close neighboring dataset (with $n {\epsilon_1}$) ] \label{lem:beta_dis_kernel_sensitivity}
If we have the below conditions,
\begin{itemize}
    \item Let $B > 0$ be a constant.
    \item Let $n$ be the number of data points.
    \item We have dataset ${\cal D} = \{(x_i, y_i)\}_{i=1}^n$, where $x_i \in \R^d$ and $\|x_i\|_2 \le B$ for any $i \in [n]$.
    \item Let the neighbor dataset ${\cal D'}$ be defined Definition~\ref{def:beta_ntk_neighbor_dataset}. 
    \item Let the discrete quadratic kernel matrices $H^{\mathrm{dis}} \in \mathbb{R}^{n \times n}$ and ${H^{\mathrm{dis}}}' \in \mathbb{R}^{n \times n}$ be defined as Definition~\ref{def:dis_quadratic_ntk}, where ${H^{\mathrm{dis}}}' \in \R^{n \times n}$ is the kernel corresponding to ${\cal D'}$. Without loss of generality, we have ${\cal D}$ and ${\cal D'}$ only differ in the $n$-th item. 
    \item Let $\delta_1, \delta_2 , \delta_3 \in (0, 1)$. Let $\delta_1 = \delta_2 / \poly(m)$. Let $\delta_2 = \delta_3 / \poly(n)$. 
    \item Let $d = \Omega (\log (1 / \delta_1))$ 
    \item Let $m = \Omega ({\epsilon_1}^{-2} d B^8 \sigma^4  \log (1 / \delta_2))$
\end{itemize}

Then, with probability $1 - \delta_3$, we can show that
\begin{align*}
    \| H^{\mathrm{dis}} - {H^{\mathrm{dis}}}' \|_F \leq O (n {\epsilon_1} + \sqrt{n} \sigma^2 B^3 \beta)
\end{align*}

\end{lemma}

\begin{proof}

Recall in Lemma~\ref{lem:union_bound_of_h_dis_h_cts_f_norm}, Let $\delta_2 , \delta_3 \in (0, 1)$. Let $\delta_2 = \delta_3 / \poly(n)$. Then, if we choose $m = \Omega ({\epsilon_1}^{-2} d B^8 \sigma^4  \log (1 / \delta_2))$, then, with probability $1 - \delta_3$, we have
\begin{align} \label{eq:beta_h_dis_h_cts_f_norm_diff_n_eps}
    \|H^{\mathrm{dis}} - H^{\mathrm{cts}}\|_F \leq n {\epsilon_1}
\end{align}

By Lemma~\ref{lem:beta_cts_quadratic_sensitivity}, we have
\begin{align} \label{eq:beta_h_cts_sensitivity}
    \| H^{\mathrm{cts}} - {H^{\mathrm{cts}}}' \|_F \leq C \sqrt{n} \sigma^2 B^3 \beta
\end{align}

Then, we have the following
\begin{align*}
    \| H^{\mathrm{dis}} - {H^{\mathrm{dis}}}' \|_F = & ~ \| ( H^{\mathrm{dis}} -  H^{\mathrm{cts}} ) + ( H^{\mathrm{cts}} - {H^{\mathrm{cts}}}' ) + ( {H^{\mathrm{cts}}}' - {H^{\mathrm{dis}}}' ) \|_F \\
    \leq & ~ \|  H^{\mathrm{dis}} -  H^{\mathrm{cts}} \|_F + \| H^{\mathrm{cts}} - {H^{\mathrm{cts}}}' \|_F + \| {H^{\mathrm{cts}}}' - {H^{\mathrm{dis}}}' \|_F \\
    \leq & ~ 2n{\epsilon_1} + \| H^{\mathrm{cts}} - {H^{\mathrm{cts}}}' \|_F \\
    \leq & ~ C (n {\epsilon_1} + \sqrt{n} \sigma^2 B^3 \beta)
\end{align*}
where the 1st step is because of basic algebra, the 2nd step is due to $\| a + b \|_F \leq \| a \|_F + \| b \|_F$ holds for any $a, b \in \mathbb{R}^{d \times d}$, the 3rd step is from Eq.~\eqref{eq:beta_h_dis_h_cts_f_norm_diff_n_eps}, the 4th step comes from Eq.~\eqref{eq:beta_h_cts_sensitivity}. 

\end{proof}

We choose an appropriate ${\epsilon_1}$ to have the final sensitivity of the Discrete Quadratic NTK Matrix.

\begin{lemma} [Sensitivity of Discrete Quadratic NTK under $\beta$-close neighboring dataset ]
\label{lem:beta_choice_of_eps}
    If we have the below conditions,
\begin{itemize}
    \item Let $B > 0$ be a constant.
    \item Let $n$ be the number of data points.
    \item We have dataset ${\cal D} = \{(x_i, y_i)\}_{i=1}^n$, where $x_i \in \R^d$ and $\|x_i\|_2 \le B$ for any $i \in [n]$.
    \item Let the neighbor dataset ${\cal D'}$ be defined Definition~\ref{def:beta_ntk_neighbor_dataset}. 
    \item Let the discrete quadratic kernel matrices $H^{\mathrm{dis}} \in \mathbb{R}^{n \times n}$ and ${H^{\mathrm{dis}}}' \in \mathbb{R}^{n \times n}$ be defined as Definition~\ref{def:dis_quadratic_ntk}, where ${H^{\mathrm{dis}}}' \in \R^{n \times n}$ is the kernel corresponding to ${\cal D'}$. Without loss of generality, we have ${\cal D}$ and ${\cal D'}$ only differ in the $n$-th item. 
    \item Let $\delta_1, \delta_2 , \delta_3 \in (0, 1)$. Let $\delta_1 = \delta_2 / \poly(m)$. Let $\delta_2 = \delta_3 / \poly(n)$.  
    \item Let $d = \Omega (\log (1 / \delta_1))$.  
    \item Let $m = \Omega(n \cdot d B^2 \beta^{-2} \log (1 / \delta_2))$.  
    \item Let ${\epsilon_1} = O(n^{-1/2} \sigma^2 B^3 \beta)$. 
\end{itemize}

Then, with probability $1 - \delta_3$, we can show that
\begin{align*}
    \| H^{\mathrm{dis}} - {H^{\mathrm{dis}}}' \|_F \leq O (\sqrt{n} \sigma^2 B^3 \beta)
\end{align*}

\end{lemma}

\begin{proof}

By Lemma~\ref{lem:beta_dis_kernel_sensitivity}, with probability $1 - \delta_3$, we have
\begin{align*}
    \| H^{\mathrm{dis}} - {H^{\mathrm{dis}}}' \|_F \leq O (n {\epsilon_1} + \sqrt{n} \sigma^2 B^3 \beta)
\end{align*}

Since we choose ${\epsilon_1} = O(n^{-1/2} \sigma^2 B^3 \beta)$, we have
\begin{align*}
    O(n {\epsilon_1}) = O ( \sqrt{n} \sigma^2 B^3 \beta)
\end{align*}

Also, since we choose ${\epsilon_1} = O(n^{-1/2} \sigma^2 B^3 \beta)$, we have
\begin{align*}
    m = \Omega(n \cdot d B^2 \beta^{-2} \log (1 / \delta_2))
\end{align*}

Then, we are done. 

\end{proof}

\subsection{Sensitivity for Privacy Guarantees} \label{sec:sensitivity_calculation}

Recall that, we have defined $M$ in Definition~\ref{def:m}. And we have defined neighboring dataset $\mathcal{D}$ and $\mathcal{D}'$ in Definition~\ref{def:beta_ntk_neighbor_dataset}. 

Then, we have 
\begin{align*}
    M := \| \mathcal{M}({\cal D})^{1/2}\mathcal{M}({\cal D}')^{-1}\mathcal{M}({\cal D})^{1/2} - I_{n \times n}  \|_F 
\end{align*}

In this section, we demonstrate that ${\cal M}(D) = H^{\mathrm{dis}}$ satisfies the assumption specified in {\bf Condition 5} of Theorem~\ref{lem:results_of_the_aussian_sampling_mechanism:formal} for ${\cal M}(D)$. Namely, we want to prove that $M \leq \Delta$, for some $\Delta$. 

In order to calculate $M$, we need the PSD inequality tools, which is proved in the following Lemma. 

\begin{lemma} [PSD Inequality] \label{lem:sensitivity}
If we have the below conditions,
\begin{itemize}
    \item Let  ${\cal D} \in \R^{n \times d}$ and ${\cal D}' \in \R^{n \times d}$ are neighboring dataset (see Definition~\ref{def:beta_ntk_neighbor_dataset})
    \item Let $H^{\mathrm{dis}}$ denotes the discrete NTK kernel matrix generated by ${\cal D}$, and ${H^{\mathrm{dis}}}'$ denotes the discrete NTK kernel matrix generated by neighboring dataset ${\cal D}'$.
    \item Let $H^{\mathrm{dis}} \succeq \eta_{\min} I_{n \times n}$, for some fixed $\eta_{\min} \in \R$. 
    \item Let $\beta = O(\eta_{\min} / \poly(n, \sigma, B))$, where $\beta$ is defined in Definition~\ref{def:beta_ntk_neighbor_dataset}. 
    \item Let $\psi := O (\sqrt{n} \sigma^2 B^3 \beta)$. 
    \item Let $\delta_1, \delta_2 , \delta_3 \in (0, 1)$. Let $\delta_1 = \delta_2 / \poly(m)$. Let $\delta_2 = \delta_3 / \poly(n)$.  
    \item Let $d = \Omega (\log (1 / \delta_1))$.  
    \item Let $m = \Omega(n \cdot d B^2 \beta^{-2} \log (1 / \delta_2))$.  
\end{itemize}
Then, with probability $1 - \delta_3$, we have
\begin{align*}
     (1 - \psi / \eta_{\min}) H^{\mathrm{dis}} \preceq {H^{\mathrm{dis}}}' \preceq (1 + \psi / \eta_{\min}) H^{\mathrm{dis}}
\end{align*}
\end{lemma}

\begin{proof}

We need to calculate the spectral norm. Namely, we need  $\| {H^{\mathrm{dis}}}' - H^{\mathrm{dis}} \|$. 

By Fact~\ref{fac:norm}, we have
\begin{align*}
    \| {H^{\mathrm{dis}}}' - H^{\mathrm{dis}} \| \leq \| {H^{\mathrm{dis}}}' - H^{\mathrm{dis}} \|_F
\end{align*}

By Lemma~\ref{lem:beta_choice_of_eps}, we have 
\begin{align*}
    \| H^{\mathrm{dis}} - {H^{\mathrm{dis}}}' \|_F \leq O (\sqrt{n} \sigma^2 B^3 \beta)
\end{align*}

Combining the two analysis mentioned above, we have
\begin{align} \label{eq:h_dis_spectral_norm_bound}
    \| H^{\mathrm{dis}} - {H^{\mathrm{dis}}}' \| \leq O (\sqrt{n} \sigma^2 B^3 \beta)
\end{align}

Since we have $\psi = O (\sqrt{n} \sigma^2 B^3 \beta)$ and $\beta = O(\eta_{\min} / \poly(n, \sigma, B))$. 

Therefore, we have 
\begin{align*}
    \psi / \eta_{\min} < 1
\end{align*}

Then, we have
\begin{align*}
    {H^{\mathrm{dis}}}' \succeq & ~ H^{\mathrm{dis}} - \psi I_{n \times n} \\
    \succeq & ~ (1 - \psi / \eta_{\min} )H^{\mathrm{dis}}
\end{align*}
where the 1st step is because of Eq.~\eqref{eq:h_dis_spectral_norm_bound}, the 2nd step is due to $H^{\mathrm{dis}} \succeq \eta_{\min} I_{n \times n}$. 

\end{proof}

Then, we are ready to calculate the exact value of $M$ based on the Lemma proved above. 

\begin{lemma} [Sensitivity of PSD Matrix of $H^{\mathrm{dis}}$, formal version of Lemma~\ref{lem:sensitivity_from_spectral_to_F}] \label{lem:sensitivity_from_spectral_to_F:formal}
If we have the below conditions,

\begin{itemize}
    \item If  ${\cal D} \in \R^{n \times d}$ and ${\cal D}' \in \R^{n \times d}$ are neighboring dataset (see Definition~\ref{def:beta_ntk_neighbor_dataset})
    \item Let $H^{\mathrm{dis}}$ denotes the discrete NTK kernel matrix generated by ${\cal D}$, and ${H^{\mathrm{dis}}}'$ denotes the discrete NTK kernel matrix generated by neighboring dataset ${\cal D}'$.
    \item Let $H^{\mathrm{dis}} \succeq \eta_{\min} I_{n \times n}$, for some $\eta_{\min} \in \R$. 
    \item Let $\beta = O(\eta_{\min} / \poly(n, \sigma, B))$, where $\beta$ is defined in Definition~\ref{def:beta_ntk_neighbor_dataset}. 
    \item Let $\psi := O (\sqrt{n} \sigma^2 B^3 \beta )$. 
    \item Let $\delta_1, \delta_2 , \delta_3 \in (0, 1)$. Let $\delta_1 = \delta_2 / \poly(m)$. Let $\delta_2 = \delta_3 / \poly(n)$.  
    \item Let $d = \Omega (\log (1 / \delta_1))$.  
    \item Let $m = \Omega(n \cdot d B^2 \beta^{-2} \log (1 / \delta_2))$.  
\end{itemize}

Then, with probability $1 - \delta_3$, we have 
\begin{itemize}
\item Part 1.
\begin{align*}
  \|  (H^{\mathrm{dis}})^{-1/2}  {H^{\mathrm{dis}}}' (H^{\mathrm{dis}})^{-1/2} - I \| \leq \psi / \eta_{\min}
\end{align*}
\item Part 2.
\begin{align*}
     \|  (H^{\mathrm{dis}})^{-1/2}  {H^{\mathrm{dis}}}' (H^{\mathrm{dis}})^{-1/2} - I \|_F \leq \sqrt{n} \psi / \eta_{\min}
\end{align*}
\end{itemize}
\end{lemma}

\begin{proof}

{\bf Proof of Part 1.}
By Lemma~\ref{lem:sensitivity}, we have
\begin{align*}
     (1 - \psi / \eta_{\min}) H^{\mathrm{dis}} \preceq {H^{\mathrm{dis}}}' \preceq (1 + \psi / \eta_{\min}) H^{\mathrm{dis}}
\end{align*}

which implies
\begin{align*}
     (1 - \psi / \eta_{\min}) I_{n \times n} \preceq (H^{\mathrm{dis}})^{1/2} {H^{\mathrm{dis}}}' (H^{\mathrm{dis}})^{1/2} \preceq (1 + \psi / \eta_{\min}) I_{n \times n} 
\end{align*}

which implies
\begin{align*}
      - \psi / \eta_{\min} I_{n \times n} \preceq (H^{\mathrm{dis}})^{1/2} {H^{\mathrm{dis}}}' (H^{\mathrm{dis}})^{1/2} - I_{n \times n} \preceq \psi / \eta_{\min} I_{n \times n}
\end{align*}

which implies
\begin{align} \label{eq:h_dis_minus_I_spectral_norm_bound}
  \|  (H^{\mathrm{dis}})^{-1/2}  {H^{\mathrm{dis}}}' (H^{\mathrm{dis}})^{-1/2} - I \| \leq \psi / \eta_{\min}
\end{align}

{\bf Proof of Part 2.}
By Fact~\ref{fac:norm}, for any $A \in \R^{n \times n}$, we have
\begin{align} \label{eq:f_norm_less_sqrt_spectral_norm}
    \| A \|_F \leq \sqrt{n} \| A \| 
\end{align}

Combining Eq.~\eqref{eq:h_dis_minus_I_spectral_norm_bound} and Eq.~\eqref{eq:f_norm_less_sqrt_spectral_norm}, we have
\begin{align*}
     \|  (H^{\mathrm{dis}})^{-1/2}  {H^{\mathrm{dis}}}' (H^{\mathrm{dis}})^{-1/2} - I \|_F \leq \sqrt{n} \psi / \eta_{\min}
\end{align*}

We can choose $\beta$ sufficiently small, such that $\psi / \eta_{\min} < 1$. 

\end{proof}

\section{Utility Guarantees for \texorpdfstring{$(K+\lambda I)^{-1}$}{}}
\label{appendix:sec:utility_guarantees}

In this section, we establish the utility guarantees for our algorithm. Our proof involves employing the spectral norm of $H^{\mathrm{dis}} - \wt{H}^{\mathrm{dis}}$ as a pivotal element. Initially, we calculate the spectral norm of $(K + \lambda I)^{-1}$. Subsequently, we determine the $\mathcal{L}_2$ norms for $\mathsf{K}(x,X)$ and $Y$, ultimately leading to the assessment of our algorithm's utility.

We begin by calculating the spectral norm of $H^{\mathrm{dis}} - \wt{H}^{\mathrm{dis}}$. 

\begin{lemma} [ Spectral norm of $H^{\mathrm{dis}} - \wt{H}^{\mathrm{dis}}$] \label{lem:spectral_norm_of_h_dis_and_wt_h_dis}
If we have the below conditions,
\begin{itemize}
    \item If  ${\cal D} \in \R^{n \times d}$ and ${\cal D}' \in \R^{n \times d}$ are neighboring dataset (see Definition~\ref{def:beta_ntk_neighbor_dataset})
    \item Let $H^{\mathrm{dis}}$ denotes the discrete NTK kernel matrix generated by ${\cal D}$ (see Definition~\ref{def:dis_quadratic_ntk}).
    \item Let $ \eta_{\max} I_{n \times n} \succeq H^{\mathrm{dis}} \succeq \eta_{\min} I_{n \times n}$, for some $\eta_{\max}, \eta_{\min} \in \R$.
    \item Let $\wt{H}^{\mathrm{dis}}$ denote the private $H^{\mathrm{dis}}$ generated by Algorithm~\ref{alg:the_gaussian_sampling_mechanism} with $H^{\mathrm{dis}}$ as input. 
    \item Let $\rho = O( \sqrt{ ( n^2+\log(1/\gamma) )  / k }+ ( n^2+\log(1/\gamma) ) /{k} )$. 
    \item Let $\gamma \in (0, 1)$. 
\end{itemize}

Then, with probability $1 - \gamma$, we have
\begin{align*}
    \| H^{\mathrm{dis}} - \wt{H}^{\mathrm{dis}} \| \leq \rho \cdot \eta_{\max}
\end{align*}
\end{lemma}

\begin{proof}
By Part 3 of Theorem~\ref{lem:results_of_the_aussian_sampling_mechanism:formal}, with probability $1 - \gamma$, we have
\begin{align*}
(1-\rho) H^{\mathrm{dis}} \preceq \wt{H}^{\mathrm{dis}} \preceq (1+\rho)  H^{\mathrm{dis}}
\end{align*}  
which implies
\begin{align} \label{eq:wt_h_dis_psd_bound}
- \rho H^{\mathrm{dis}} \preceq \wt{H}^{\mathrm{dis}} -  H^{\mathrm{dis}} \preceq \rho  H^{\mathrm{dis}}
\end{align}

Then, we have
\begin{align*}
    \| \wt{H}^{\mathrm{dis}} -  H^{\mathrm{dis}} \| \leq \rho \cdot \eta_{\max}
\end{align*}

\end{proof}

To calculate the spectral norm of the inverse of the difference between two matrices, we need to involve a classical technique from the literature, where
\cite{w73} presented a perturbation bound of Moore-Penrose inverse the spectral norm,
\begin{lemma}[\cite{w73}, Theorem 1.1 in \cite{mz10}]
\label{lem:inverse_minus_bound}
Given two matrices $A,B \in \R^{d_1 \times d_2}$ with full column rank, we have
\begin{align*}
\| A^\dagger - B^\dagger \| \lesssim \max ( \| A^\dagger \|^2, \| B^\dagger \|^2 ) \cdot \| A - B \|.
\end{align*}
\end{lemma}

With the help of the perturbation bound, we can now reach the spectral norm of $(K + \lambda I)^{-1}$. 

\begin{lemma} 
[Utility guarantees for $(K+\lambda I)^{-1}$, formal version of Lemma~\ref{lem:k_lambda_inverse_utility:informal}]
\label{lem:k_lambda_inverse_utility}
If we have the below conditions,
\begin{itemize}
    \item If  ${\cal D} \in \R^{n \times d}$ and ${\cal D}' \in \R^{n \times d}$ are neighboring dataset (see Definition~\ref{def:beta_ntk_neighbor_dataset})
    \item Let $H^{\mathrm{dis}}$ denotes the discrete NTK kernel matrix generated by ${\cal D}$ (see Definition~\ref{def:dis_quadratic_ntk}).
    \item Let $\wt{H}^{\mathrm{dis}}$ denotes the private $H^{\mathrm{dis}}$ generated by Algorithm~\ref{alg:the_gaussian_sampling_mechanism} with $H^{\mathrm{dis}}$ as the input. 
    \item Let $ \eta_{\max} I_{n \times n} \succeq H^{\mathrm{dis}} \succeq \eta_{\min} I_{n \times n}$, for some $\eta_{\max}, \eta_{\min} \in \R$.
    \item Let $K = H^{\mathrm{dis}}, \wt{K} = \wt{H}^{\mathrm{dis}}$. 
    \item Let $\lambda \in \R$. 
    \item Let $\rho = O( \sqrt{ ( n^2+\log(1/\gamma) )  / k }+ ( n^2+\log(1/\gamma) ) /{k} )$. 
    \item Let $\gamma \in (0, 1)$. 
\end{itemize}

Then, with probability $1 - \gamma$, we have
\begin{align*}
    \|(K + \lambda I)^{-1} -  (\wt{K} + \lambda I)^{-1} \| \leq 
    O( \frac{\rho \cdot \eta_{\max}}{(\eta_{\min} + \lambda)^2}) 
\end{align*}

\end{lemma}

\begin{proof}
We consider the $\| (K + \lambda I)^{-1} \|$ term. We have
\begin{align} \label{eq:k_lambda_inverse_spectral}
    \| (K + \lambda I)^{-1} \| = & ~ \sigma_{\max} ((K + \lambda I)^{-1}) \notag \\
    = & ~ \frac{1}{\sigma_{\min} ((K + \lambda I))} \notag \\
    \leq & ~ \frac{1}{\eta_{\min} + \lambda}
\end{align}
where the 1st step is because of definition of spectral norm, the 2nd step is due to $\sigma_{\max}(A^{-1}) = 1 / \sigma_{\min}(A)$ holds for any matrix $A$, the 3rd step is from $K \succeq \eta_{\min} I_{n \times n}$. 

Similarly, we can have
\begin{align} \label{eq:wt_k_lambda_inverse_spectral}
    \| (\wt{K} + \lambda I)^{-1} \| \leq \frac{1}{\eta_{\min} + \lambda}
\end{align}

Recall in Lemma~\ref{lem:spectral_norm_of_h_dis_and_wt_h_dis}, we have
\begin{align} \label{eq:h_dis_spectral_bound_by_eta_and_omega}
    \| H^{\mathrm{dis}} - \wt{H}^{\mathrm{dis}} \| \leq \rho \cdot \eta_{\max}
\end{align}

Then, by Lemma~\ref{lem:inverse_minus_bound}, we have
\begin{align*}
\|(K + \lambda I)^{-1} -  (\wt{K} + \lambda I)^{-1} \| \leq & ~ O( \max ( \| (K + \lambda I)^{-1} \|^2, \| (\wt{K} + \lambda I)^{-1} \|^2 ) \cdot \| K - \wt{K} \|) \\
\leq & ~ O( \frac{1}{(\eta_{\min} + \lambda)^2} \cdot \| K - \wt{K} \|) \\
\leq & ~ O( \frac{\rho \cdot \eta_{\max}}{(\eta_{\min} + \lambda)^2}) 
\end{align*}
where the 1st step is because of Lemma~\ref{lem:inverse_minus_bound}, the 2nd step is due to Eq.~\eqref{eq:k_lambda_inverse_spectral} and Eq.~\eqref{eq:wt_k_lambda_inverse_spectral}, the 3rd step is from Eq.~\eqref{eq:h_dis_spectral_bound_by_eta_and_omega}. 
\end{proof}

\section{DP Guarantees for \texorpdfstring{$\mathsf{K}(x, X)$}{}} \label{sec:app:Kxx_dp}

In this section, we will discuss how to add truncated Laplace noise $\TLap(\Delta, \epsilon, \delta)$ on the training data $X \in \R^{n \times d}$ to ensure the differentially private property of $\K(x, X)$. 

Firstly, we need to analyze the sensitivity of $X$.

\begin{lemma} [Sensitivity of $X$] \label{lem:X_sensitivity}
If the following conditions hold:
\begin{itemize}
    \item Let $X \in \R^{n \times d}$ denote the training data.
    \item Let the neighboring dataset $X$ and $X'$ be defined as Definition~\ref{def:beta_ntk_neighbor_dataset}. 
    \item Let $\beta > 0$ be defined as Definition~\ref{def:beta_ntk_neighbor_dataset}. 
    \item Let $\Delta_X := \| X - X' \|_1$ denote the sensitivity of $X$. 
\end{itemize}

Then, we can show that the sensitivity of $X$ is $\sqrt{d} \cdot \beta$. Namely, we have
\begin{align*}
    \Delta_X = \sqrt{d} \cdot \beta.
\end{align*}
\end{lemma}

\begin{proof}
Without loss of generality, we use $x_n \in \R^d$ and $x_n' \in \R^d$ to denote the different items in $X$ and $X'$.

By the definition of the neighboring dataset, we have
\begin{align*}
    \| x_n - x_n' \|_2 \leq \beta.
\end{align*}

Then, we have
\begin{align*}
    \| X - X' \|_1 = & ~ \| x_n - x_n' \|_1 \\
    \leq & ~ \sqrt{d} \cdot \| x_n - x_n' \|_2 \\
    = & ~ \sqrt{d} \cdot \beta,
\end{align*}
where the first step follows from $\|u - v\|_1 \leq \sqrt{d} \| u - v \|_2$ for any $u, v \in \R^d$, the second step follows from $\| x_n - x_n' \|_2 \leq \beta$. 
\end{proof}

Then, we use the truncated Laplace mechanism to ensure the DP property of $X$. 

\begin{lemma} [DP guarantees for $X$] \label{lem:DP_for_X}
If the following conditions hold:
\begin{itemize}
    \item Let the truncated Laplace $\TLap(\Delta, \epsilon, \delta)$ be defined as Lemma~\ref{lem:truncated_laplace_mechanism}.
    \item Let $X \in \R^{n \times d}$ denote the training data.
    \item Let $\epsilon_X > 0, \delta_X \geq 0$ denote the DP parameters for $X$. 
    \item Let $\Delta_X \geq 0$ denote the sensitivity of $X$.  
    \item Let $\wt{X} := X + \TLap(\Delta_X, \epsilon_X, \delta_X)$. 
\end{itemize}

Then, we can show that $\wt{X}$ is $(\epsilon_X, \delta_X)$-DP. 
\end{lemma}

\begin{proof}
    The proof follows directly from Lemma~\ref{lem:truncated_laplace_mechanism}. 
\end{proof}

Finally, we use post-processing Lemma to prove the DP guarantees for $\K(x, X)$. 

\begin{lemma} [DP guarantees for $\K(x, X)$, formal version of Lemma~\ref{lem:DP_for_KxX:informal}] \label{lem:DP_for_KxX}
If the following conditions hold:
\begin{itemize}
    \item Let $x \in \R^d$ denote an arbitrary query. 
    \item Let $\epsilon_X, \delta_X \in \R$ denote the DP parameters. 
    \item Let $\Delta_X := \sqrt{d} \beta$ denote the sensitivity of $X$. 
    \item Let $\K(x, X)$ be defined as Definition~\ref{def:ntk_regression}. 
    \item Let $\wt{X} := X + \TLap(\Delta_X, \epsilon_X, \delta_X)$ denote the private version of $X$, where $\wt{X}$ is $(\epsilon_X, \delta_X)$-DP. 
\end{itemize}

Then, we can show that $\K(x, \wt{X})$ is $(\epsilon_X, \delta_X)$-DP. 
\end{lemma}

\begin{proof}
Since we only care about the sensitive information in $X$, for a fixed query $x$, we can view $\K(x, X)$ as a function of $X$. Namely, for a fixed query $x$, $\K(x, X) = F(X)$. 

Then, directly follows from the post-processing Lemma~\ref{lem:post_processing_dp}, since $\wt{X}$ is $(\epsilon_X, \delta_X)$-DP, we can show that $\K(x, \wt{X})$ is $(\epsilon_X, \delta_X)$-DP. 
\end{proof}

\section{Utility Guarantees for \texorpdfstring{$\K(x, X)$}{}} \label{sec:app:KxX_utility}

In this section, we will analyze the utility of guarantees of $\K(x, \wt{X})$. 

\begin{lemma} [Utility guarantees for $\K(x, X)$, formal version of Lemma~\ref{lem:KxX_utility:informal}] \label{lem:KxX_utility}
If the following conditions hold:
\begin{itemize}
    \item Let $x \in \R^d$ be a query, where for some $B \in \R$, $\| x \|_2 \leq B$. 
    \item Let $\K(x, X) \in \R^n$ be defined as Definition~\ref{def:ntk_regression}.
    \item Let $\wt{X} \in \R^{n \times d}$ be defined as Lemma~\ref{lem:DP_for_X}. 
    \item Let $\Delta_X = \sqrt{d} \cdot \beta$. 
    \item Let $\epsilon_X, \delta_X \in \R$ denote the DP parameters for $X$. 
    \item Let $B_L =  (\Delta_X / \epsilon_X) \log (1+ \frac{\exp(\epsilon_X) - 1}{2 \delta_X})$. 
\end{itemize}

Then, we can show that
\begin{align*}
    \| \K(x, \wt{X}) - \K(x, X) \|_2 
    \leq & ~ 2 \sqrt{n} B^3 \sqrt{d} B_L.
\end{align*}
\end{lemma}

\begin{proof}
For $i \in [n], j \in [d]$, let $X(i, j), \wt{X}(i, j) \in \R$ denote the $(i, j)$-th entry of $X$ and $X'$, respectively.
Let $X_i \in \R^d$ denote the $i$-th column of $X$. 
Let $\K(x, X)_i \in \R$ denote the $i$-th entry of $\K(x, X)$.

By the definition of $\wt{X}$, we have
\begin{align*}
    \wt{X}(i, j) = X(i, j) + \TLap(\Delta_X, \epsilon_X, \delta_X)
\end{align*}

Recall that we have $B_L =  (\Delta_X / \epsilon_X) \log (1+ \frac{e^{\epsilon_X} - 1}{2 \delta_X})$. By the definition of truncated Laplace, we have
\begin{align*}
    | \TLap(\Delta_X, \epsilon_X, \delta_X) | \leq B_L
\end{align*}

Combining the above two equations, for $i \in [n]$, we have
\begin{align} \label{eq:wtXi_bound}
    \| \wt{X}_i - X_i \|_2 \leq \sqrt{d} \cdot B_L
\end{align}

We consider $\K(x, X)_i$. By the Lipchisz property of $\K(x, X)$ (Lemma~\ref{lem:quadratic_ntk_non_diagonal_single_lipschitz}), we have
\begin{align} \label{eq:KxwtX_bound}
    | \K(x, \wt{X})_i - \K(x, X)_i | \leq & ~ 2 \sigma^2 B^3 \| \wt{X}_i - X_i \|_2 \notag \\
    \leq & ~ 2 \sigma^2 B^3 \sqrt{d} B_L
\end{align}
where the first step follows from Lemma~\ref{lem:quadratic_ntk_non_diagonal_single_lipschitz}, the second step follows from Eq.~\eqref{eq:wtXi_bound}. 

Then, we have
\begin{align*}
    \| \K(x, \wt{X}) - \K(x, X) \|_2 
    = & ~ (\sum_{i=1}^n (\K(x, \wt{X})_i - \K(x, X)_i)^2)^{1/2} \\
    \leq & ~ \sqrt{n} \cdot \max_{i \in [n]} | \K(x, \wt{X})_i - \K(x, X)_i | \\
    \leq & ~ 2 \sqrt{n} \sigma^2 B^3 \sqrt{d} B_L,
\end{align*}
where the first step follows from the definition of $\ell_2$ norm, the second step follows from basic algebra, the third step follows from Eq.~\eqref{eq:KxwtX_bound}. 

In our setting, we choose $\sigma = 1$. Then, we have our results. 

\begin{align*}
    \| \K(x, \wt{X}) - \K(x, X) \|_2 
    \leq & ~ 2 \sqrt{n} B^3 \sqrt{d} B_L.
\end{align*}

\end{proof}

\section{DP Guarantees for NTK Regression} \label{sec:app:ntk_regression_dp}

This section will prove the DP guarantees for the entire NTK regression model.

\begin{lemma} [DP guarantees for NTK regression, formal version of Lemma~\ref{lem:ntk_regerssion_dp:informal}] \label{lem:ntk_regerssion_dp}
If the following conditions hold:
\begin{itemize}
    \item Let $\epsilon_X, \delta_X \in \R$ denote the DP parameter for $\K(x, X)$. 
    \item Let $\epsilon_{\alpha}, \delta_{\alpha} \in \R$ denote the DP parameter for $(K+\lambda I)^{-1}$.
    \item Let $\epsilon = \epsilon_X + \epsilon_{\alpha}, \delta = \delta_X + \delta_{\alpha}$. 
    \item Let $\K(x, \wt{X})$ be defined as Lemma~\ref{lem:DP_for_X}. 
    \item Let $(\wt{K} + \lambda I)^{-1}$ be defined as Lemma~\ref{lem:results_of_the_aussian_sampling_mechanism:formal}. 
\end{itemize}

Then, we can show that the private NTK regression (Algorithm~\ref{alg:main}) is $(\epsilon, \delta)$-DP. 
\end{lemma}

\begin{proof}

Let $(\K(x, \wt{X}), (\wt{K} + \lambda I)^{-1})$ denote the two-tuple of $\K(x, \wt{X})$ and $(\wt{K} + \lambda I)^{-1}$. 

Since we have
\begin{itemize}
    \item $\K(x, \wt{X})$ is $(\epsilon_X, \delta_X)$-DP. 
    \item $ (\wt{K} + \lambda I)^{-1}$ is $(\epsilon_{\alpha}, \delta_{\alpha})$-DP. 
    \item $\epsilon = \epsilon_X + \epsilon_{\alpha}, \delta = \delta_X + \delta_{\alpha}$
\end{itemize}
Then, by the composition lemma of DP (Lemma~\ref{lem:dp_composition}), we have $(\K(x, \wt{X}), (\wt{K} + \lambda I)^{-1})$ is $(\epsilon, \delta)$-DP. 

Since we aim to protect the sensitive information in $X$, then the NTK regression can be viewed as a function that takes $(\K(x, \wt{X}), (\wt{K} + \lambda I)^{-1})$ as the input. 

Then, by the post-processing lemma (Lemma~\ref{lem:post_processing_dp}), we have the private NTK regression (Algorithm~\ref{alg:main}) is $(\epsilon, \delta)$-DP.

\end{proof}

\section{Utility Guarantees for NTK Regression} \label{sec:app:ntk_regression_utility}

Then, we can finally reach the utility guarantees for the private NTK regression, which is discussed in the following Lemma.

\begin{lemma} [Utility guarantees for NTK regression, formal version of Lemma~\ref{lem:ntk_regression_utility:informal}]\label{lem:ntk_regression_utility}
If we have the below conditions,
\begin{itemize}
    \item If  ${\cal D} \in \R^{n \times d}$ and ${\cal D}' \in \R^{n \times d}$ are neighboring dataset (see Definition~\ref{def:beta_ntk_neighbor_dataset})
    \item Let $H^{\mathrm{dis}}$ denotes the discrete NTK kernel matrix generated by ${\cal D}$ (see Definition~\ref{def:dis_quadratic_ntk}). 
    \item Let $ \eta_{\max} I_{n \times n} \succeq H^{\mathrm{dis}} \succeq \eta_{\min} I_{n \times n}$, for some $\eta_{\max}, \eta_{\min} \in \R$.
    \item Let $\wt{H}^{\mathrm{dis}}$ denotes the private $H^{\mathrm{dis}}$ generated by Algorithm~\ref{alg:the_gaussian_sampling_mechanism} with $H^{\mathrm{dis}}$ as the input. 
    \item Let $K = H^{\mathrm{dis}}, \wt{K} = \wt{H}^{\mathrm{dis}}$ in Definition~\ref{def:ntk_regression}. Then we have $f_{K}^*(x)$ and $f_{\wt{K}}^*(x)$. 
    \item Let $\sqrt{n} \psi / \eta_{\min} < \Delta$, where $\Delta$ is defined in Definition~\ref{def:delta}. 
    \item Let $\rho = O( \sqrt{ ( n^2+\log(1/\gamma) )  / k }+ ( n^2+\log(1/\gamma) ) /{k} )$. 
    \item Let $\omega :=  6 d \sigma^2 B^4$. 
    \item Let $B_L \in \R$ be defined in Lemma~\ref{lem:truncated_laplace_mechanism}. 
    \item Let $\gamma \in (0, 1)$. 
\end{itemize}

Then, with probability $1 - \gamma$, we have 
\begin{align*}
    | f_{K}^*(x) - f_{\wt{K}}^*(x) | \leq 
    O(\frac{B^3 \sqrt{d} B_L}{\eta_{\min} + \lambda} + \frac{ \rho \cdot \eta_{\max} \cdot \omega}{(\eta_{\min} + \lambda)^2 })
\end{align*}

\end{lemma}

\begin{proof}
We consider the $\|\mathsf{K}(x,X)\|_2$ term.

Let $\mathsf{K}(x,X)(s) \in \R$ denote the $s$ -th entry of $\mathsf{K}(x,X) \in \R^n$. Let $\Xi_r := \langle \langle w_r , x \rangle x,  \langle w_r , x_i \rangle  x_i \rangle$.

Then, for any $s \in [n]$, we have 
\begin{align} \label{eq:k_xx_s_bound}
    | \mathsf{K}(x,X)(s) | = & ~ | \frac{1}{m} \sum_{s=1}^m  \Xi_r | \notag \\
    \leq & ~  6 d \sigma^2 B^4 \notag \\
    = & ~ \omega
\end{align}
where the 1st step is because of the definition of $\mathsf{K}(x,X)$, the 2nd step is due to Lemma~\ref{lem:upper_bound_z_i}, the 3rd step is from definition of $\omega$. 

Then we have
\begin{align} \label{eq:k_xx_l2_bound}
    \| \mathsf{K}(x,X) \|_2 \leq & ~ \sqrt{n} | \mathsf{K}(x,X)(s) | \notag \\
    \leq & ~ \sqrt{n} \omega 
\end{align}
where the 1st step is because of the definition of $\| \cdot \|_2$, the 2nd step is due to Eq.~\eqref{eq:k_xx_s_bound}. 

Then, we consider the $\|Y\|_2$.

Since we have $Y \in \{ 0, 1 \}^n$, then we have
\begin{align} \label{eq:y_l2_bound}
    \|Y\|_2 \leq \sqrt{n}
\end{align}

Let $\rho = O( \sqrt{ ( n^2+\log(1/\gamma) )  / k }+ ( n^2+\log(1/\gamma) ) /{k} )$. 

By Lemma~\ref{lem:k_lambda_inverse_utility}, with probability $1 - \gamma$, we have
\begin{align} \label{eq:k_lambda_inverse_spectral_bound}
    \|(K + \lambda I)^{-1} -  (\wt{K} + \lambda I)^{-1} \| \leq O( \frac{\rho \cdot \eta_{\max}}{(\eta_{\min} + \lambda)^2})
\end{align}

We have
\begin{align} \label{eq:wtk_bound}
    & ~ | \mathsf{K}(x, X)^\top (\wt{K} + \lambda I)^{-1} Y - \mathsf{K}(x,X)^\top (K + \lambda I)^{-1} Y | \notag \\
    \le & ~ \|\mathsf{K}(x,X)\|_2 \|(K + \lambda I)^{-1} -  (\wt{K} + \lambda I)^{-1} \| \|Y\|_2 \notag \\
    \le & ~ \sqrt{n} \omega \cdot O( \frac{\rho \cdot \eta_{\max}}{(\eta_{\min} + \lambda)^2 }) \cdot \sqrt{n} \notag \\
    = & ~ O( \frac{ n \rho \cdot \eta_{\max} \cdot \omega}{(\eta_{\min} + \lambda)^2 })
\end{align}
where the first step follows from $| b^\top A c | \leq \| b \|_2 \| A \| \| c \|_2$ holds for any $b, c \in \R^d , A \in \R^{d \times d}$, the 2nd step is from Eq.~\eqref{eq:k_xx_l2_bound}, Eq.~\eqref{eq:y_l2_bound}, and Eq.~\eqref{eq:k_lambda_inverse_spectral_bound}, the 3rd step comes from basic algebra. 

By Lemma~\ref{lem:KxX_utility}, we have
\begin{align} \label{eq:wtX:bound}
    \| \K(x, \wt{X}) - \K(x, X) \|_2 \leq 2 B^3 \sqrt{n} B_L
\end{align}

Then, we have
\begin{align} \label{eq:wtx_bound}
    & ~ | \mathsf{K}(x,\wt{X})^\top (\wt{K} + \lambda I)^{-1} Y - \mathsf{K}(x, X)^\top (\wt{K} + \lambda I)^{-1} Y | \notag \\
    \leq & ~ \| \K(x, \wt{X}) - \K(x, X) \|_2 \| (\wt{K} + \lambda I)^{-1} \| \| Y \|_2 \notag \\
    \leq & ~ 2 \sqrt{n} B^3 \sqrt{d} B_L \frac{1}{\eta_{\min} + \lambda} \sqrt{n} \notag \\
    = & ~ \frac{2 n \sqrt{d} B^3 B_L}{\eta_{\min} + \lambda}
\end{align}
where the first step follows from $| b^\top A c | \leq \| b \|_2 \| A \| \| c \|_2$, the second step follows from Eq.~\eqref{eq:y_l2_bound}, Eq.~\eqref{eq:k_lambda_inverse_spectral_bound}, and Eq.~\eqref{eq:wtX:bound}, the third step follows from basic algebra. 

Eventually, we can combine the above two equations to get our final results. 
\begin{align*}
    & ~ | f_{\wt{K}}^*(x) - f_K^*(x) | \\
    = & ~ \frac{1}{n} | \mathsf{K}(x,\wt{X})^\top (\wt{K} + \lambda I)^{-1} Y - \mathsf{K}(x,X)^\top (K + \lambda I)^{-1} Y | \notag \\
    \leq & ~ \frac{1}{n} (| \mathsf{K}(x,\wt{X})^\top (\wt{K} + \lambda I)^{-1} Y - \mathsf{K}(x, X)^\top (\wt{K} + \lambda I)^{-1} Y | \\
    + & ~ | \mathsf{K}(x, X)^\top (\wt{K} + \lambda I)^{-1} Y - \mathsf{K}(x,X)^\top (K + \lambda I)^{-1} Y | ) \\
    \leq & ~ \frac{1}{n} (\frac{2 n \sqrt{d} B^3 B_L}{\eta_{\min} + \lambda} + O( \frac{ n \rho \cdot \eta_{\max} \cdot \omega}{(\eta_{\min} + \lambda)^2 })) \\
    = & ~ O(\frac{\sqrt{d} B^3 B_L}{\eta_{\min} + \lambda} + \frac{ \rho \cdot \eta_{\max} \cdot \omega}{(\eta_{\min} + \lambda)^2 })
\end{align*}
where the 1st step is because of Definition~\ref{def:ntk_regression}, the second step follows from the triangle inequality, the third step follows from Eq.~\eqref{eq:wtk_bound} and Eq.~\eqref{eq:wtx_bound}, the fourth step follows from basic algebra.

\end{proof}




\end{document}